\documentclass[twoside]{article}

\usepackage[accepted]{aistats2025}
\usepackage[round]{natbib}

\usepackage[utf8]{inputenc} 
\usepackage[T1]{fontenc}    
\usepackage{hyperref}       
\usepackage{url}            
\usepackage{booktabs}       
\usepackage{amsfonts}       
\usepackage{nicefrac}       
\usepackage{microtype}      
\usepackage{xcolor}         
\usepackage{algorithm}
\usepackage{algpseudocode}
\usepackage{graphicx}
\usepackage{wrapfig}
\usepackage{natbib}
\usepackage{ mathrsfs }

\hypersetup{
    colorlinks,
    linkcolor={red!50!black},
    citecolor={blue!60!black},
    urlcolor={blue!80!black}
}

\usepackage{YK}

\def\cvx{\mathrm{cvx}}

\usepackage[accepted]{aistats2025}

\begin{document}

%

%
\runningauthor{Bracale, Maia Polo, Maity, Somerstep, Banerjee, Sun}

\twocolumn[

\aistatstitle{Microfoundation Inference for Strategic Prediction}

\aistatsauthor{ Daniele Bracale$^*$ \And Felipe Maia Polo$^*$ \And  Subha Maity$^*$ }

\aistatsaddress{ dbracale@umich.edu\\
Department of Statistics\\
University of Michigan \And maiapolo@umich.edu\\
Department of Statistics\\
University of Michigan\And smaity@uwaterloo.ca\\Department of Statistics\\\& Actuarial Science \\University of Waterloo } 

\aistatsauthor{ Seamus Somerstep$^*$ \And Moulinath Banerjee \And Yuekai Sun }

\aistatsaddress{ smrstep@umich.edu \\
Department of Statistics\\
University of Michigan \And moulib@umich.edu\\Department of Statistics\\University of Michigan \And yuekai@umich.edu\\Department of Statistics\\University of Michigan }
]

\renewcommand{\thefootnote}{} 
\footnotetext{* Equal contribution.}
\renewcommand{\thefootnote}{\arabic{footnote}} 

\begin{abstract}
Often in prediction tasks, the predictive model itself can influence the distribution of the target variable, a phenomenon termed \emph{performative prediction}. Generally, this influence stems from strategic actions taken by stakeholders with a vested interest in predictive models.
A key challenge that hinders the widespread adaptation of performative prediction in machine learning is that practitioners are generally unaware of the social impacts of their predictions. To address this gap, we propose a methodology for learning the distribution map that encapsulates the long-term impacts of predictive models on the population. Specifically, we model agents' responses as a cost-adjusted utility maximization problem and propose estimates for said cost.
Our approach leverages optimal transport to align pre-model exposure (\emph{ex ante}) and post-model exposure (\emph{ex post}) distributions. We provide a rate of convergence for this proposed estimate and assess its quality through empirical demonstrations on a credit scoring dataset.
\end{abstract}

\section{Introduction}


Consider a supervised learning scenario where the predictive model triggers actions that alter the target distribution, thereby impacting its own performance. This scenario is referred to as \emph{performative prediction} \citep{perdomo2020Performative} and is prevalent in {social} prediction contexts, where predictions influence (strategic) individuals. For example, traffic forecasts affect traffic patterns, crime location predictions influence police deployment that may deter crime, and stock price predictions drive trading activity, consequently affecting prices. Campbell's law \citep{campbell1979assessing}, which states that
\vspace{-0.2cm}
\begin{quote}
The more any quantitative social indicator is used for social decision-making, the more subject it will be to corruption pressures and the more apt it will be to distort and corrupt the social processes it is intended to monitor.
\vspace{-0.2cm}
\end{quote}
acknowledges the ubiquity of performativity in social prediction problems.


The learner may respond to this performative distribution shift by frequently retraining the model with new data, but this approach is often impractical, as \citet{perdomo2020Performative} describes this as a ``cat and mouse game of chasing a moving target.'' They suggest a more structured approach, framing risk minimization under performativity within a decision-theoretic framework to find equilibria where the model is optimal for the distribution it creates. This idea was also explored by \citet{hardt2016Strategic} in \emph{strategic classification}, a notable subtype of performative prediction. Subsequent research has focused on computing performative stable \citep{mendler-dunner2020Stochastic} or performative optimal \citep{izzo2022How} policies and developing related algorithms for online settings \citep{jagadeesan2022Regret}. For a more detailed review of the literature, see Section \ref{sec:related-works}.

In performative prediction, the specific mechanism of the distribution shift is usually unknown. Previous work circumvents this issue in two ways: The first is to treat the performative prediction problem \eqref{eq:performative-prediction} as a zeroth order (derivative-free) optimization problem, where one only relies on samples drawn from the distributions $Q_{\theta_t}$ induced by iterations $\theta_t$ and solves the performative prediction problem without knowing the \emph{performative distribution map} $\theta \mapsto Q_{\theta}$, \ie\ the map of performative distribution induced by a predictor $\theta$ \citep{izzo2021How,miller2021Outside}. One drawback of this approach is its slow convergence: \eg\ \citeauthor{izzo2021How}'s algorithm relies on finite differences to approximate the gradient of the performative risk, so it suffers from the curse of dimensionality.

The second approach \emph{stipulates} that each agent follows a (potentially misspecified) microfoundation model \citep{jang2022Sequential} and incorporates the anticipated agent responses into the optimization. This method leverages the fast convergence rates of the algorithm for performative risk minimization that benefit from full access to the distribution map \citep{hardt2016Strategic,levanon2021Strategica,cutler2021Stochastic, somerstep2023Learning}. However, the agent response model is often inaccurate \citep{lin2023plug}, leading to incorrect distribution maps and preventing effective resolution of the performative prediction problem. We illustrate this issue in an example of strategic ordinary least squares (OLS) prediction in Example \ref{ex:strategic-OLS}.

To address this issue, we propose an estimator for the microfoundation models of agent responses. Our approach complements the second approach;
on the one hand, we alleviate the issues of misspecified agent response models by learning the microfoundations model from agent response data, and on the other hand, the white-box access to the distribution map allows the learner to take advantage of faster optimization algorithms for minimizing the performative risk. Another benefit of our approach is that it is well suited to more sophisticated learning algorithms (\eg\ algorithms that enforce constraints \citep{maity2021Does,somerstep2023Learning}) that have no obvious zeroth order counterpart. 

Our approach to estimate the microfoundation model is to assume that the agents' microfoundation is a utility maximization problem consisting of a (model-dependent) known benefit function $B_f$ and an unknown cost function $c$. A detailed exposition of our approach can be found in Section \ref{sec:cost-estimation}, where we develop an optimal transport-based method to estimate the cost within a class of Bregman divergences.
Our contributions are as follows:
 
\begin{enumerate}
\item We propose a method for estimating the cost function (within a class of Bregman divergences) of a general utility-maximizing microfoundation model. Our method relies on matching pre-model deployment and post-model deployment distributions with optimal transport. 
\item We provide an analysis of identifiability in the estimation of the cost function, as well as rates of convergence of our proposed estimator.
\item  We perform numerical experiments that both exhibit the performance of our methodology in estimating the cost and demonstrate its usefulness on a variety of down-stream tasks. We also demonstrate empirically that our methodology satisfies certain robustness properties with respect to misspecification of the benefit function.  
\end{enumerate}

\subsection{Related works} 
\label{sec:related-works}

The performative prediction problem was introduced by \citet{perdomo2020Performative}. Much of the methodological focus in performative prediction is on locating performative stable or performative optimal solutions. To locate a stable solution, a paradigm of repeated retraining has been developed in multiple environments \cite{mendler-dunner2020Stochastic,drusvyatskiy2023Stochastic}. 
Furthermore, a line of prior work on this often utilizes zero'th order algorithms such as finite differences \cite{izzo2021How, izzo2022How} or regret minimization \citep{jagadeesan2022Regret, dong2018Strategic, chen2020Learning}. Other works instead assume white-box access to a pre-specified performative map \citep{levanon2021Strategica, somerstep2023Learning, shavit2020Causal, lin2023plug} or make structural assumptions on the underlying data (e.g assuming all data is drawn from a location-scale family) \citep{miller2021Outside}. The work  \citet{Tsirtsis_2024} casts finding a performatively optimal model (given access to a performative map) as an optimal transport problem.

As pointed out in \cite{lin2023plug} such pre-specified maps are often \emph{misspecified} leading to erroneous results. We focus on correcting this by building a robust method to infer performative maps.

Originating with strategic classification \cite{hardt2016Strategic}; pre-specified models of performative maps are often at the sample level, and correspond to samples behaving as agents with strategic behavior. Strategic classification has been expanded in several directions: including models with causality \cite{mendler-dunner2022Anticipating, horowitz2023Causal, somerstep2023Learning, harris2022Strategic}, models with opaque agent behavior \cite{ghalme2021Strategica}, the effects of strategic behavior on graph neural networks \cite{eilat2023strategic}, models with reversed order of play \cite{zrnic2022Who}, PAC learning for strategic classification \cite{sundaram2021paclearning},  strategic ordinary least squares \cite{shavit2020Causal} and combinations of these \citep{levanon2022Generalized}. Strategic behavior with an element of competition between agents has been used to model the behavior of content creators in socio-technical systems \citep{hron2023Modeling, jagadeesan2023SupplySidea}. The study of strategic behavior is not isolated to computer science; strategic behavior is the key ingredient to many micro-economic models of labor markets \citep{coate1993Will, moro2003Affirmative, moro2004general, fang2011Chapter}. 

The common theme of strategic behavior is that each data sample is an agent that is assumed to maximize a cost-adjusted utility that depends on the model $\theta$; therefore, the performative map is implicitly determined by the agents' utility and cost. Our main contribution is a methodology that allows practitioners to infer agents' costs and utilities. Prior work on inferring performative maps is relatively light. The works of \cite{mendler-dunner2022Anticipating, LechnerUnknown2023} assume that the aggregate manipulation structure is in a class of graphs $\mathcal{G}$ and infer the manipulation graph. \cite{shavit2020Causal} includes some focus on the inference of strategic responses, but only in the case of linear regression. Our methodology is applicable beyond these special cases and additionally allows for inference with both samples from the ex-ante and an ex-post distribution or samples from multiple ex-post distributions.

\vspace{-0.1cm}
Inference of agent costs and utilities is often a crucial step in deploying accurate models in performative settings. For example, utilizing differentiable proxies to $\argmax$, the authors of  \cite{levanon2021Strategica} provide a framework for minimizing performative risk in a strategic environment. Unfortunately, if the strategic model is not specified correctly, then the resulting estimator of an optimal model will not be consistent \citep{lin2023plug}. Our methodology allows for practitioners to utilize plug-in optimization frameworks with peace of mind.

\vspace{-0.1cm}
\section{Microfoundations for strategic prediction}

We follow the performative prediction setting in \citep[Section 5]{perdomo2020Performative}:
at the microlevel, the agents are rational, \ie\ they respond to a predictive model $\theta$ by choosing their best actions. Each agent is represented by a pair of attributes $Z = (X,Y) \in \cZ = \cX \times \cY$, where $X$ and $Y$ are the covariates and the outcome of an individual.
Before agents are exposed to a model $\theta$ their attributes follow an \emph{ex-ante distribution} $Z\sim P \in \Delta(\cZ)$ (where $\Delta (\cZ)$ is the set of probability measures over $\cZ$), and, having been exposed to $\theta$, they strategically maximize their cost-adjusted utility that depends on $\theta$; \ie{} in response to $\theta$ each agent updates their individual $Z$ via 
\vspace{-0.1cm}
\begin{equation}
T_{\theta}(Z)\in\argmax_{z'\in\cZ}B_{\theta}(z') - c(Z,z')\,.
\label{eq:agent-util-maximization}
\end{equation} Here, $B_{\theta}(z')$ encodes the benefit of the agent changing their attributes to $z'$ when exposed to $\theta$, and $c$ is a cost function that encodes the burden of changing their attributes from $Z$ to $z'$. Throughout, we will assume that \emph{$B_{\theta}(\cdot)$ is known, and $c(\cdot, \cdot)$ remains to be estimated}. Often, the utility $B_{\theta}(\cdot)$ is simply the output of the learners predictive model (e.g. in  standard strategic classification \citep{hardt2016Strategic}) or is a function of a learner controlled wage and output of the learners predictive model (e.g. in labor market models \citep{coate1993Will}). In either case, all information pertinent to the benefit function is known to the learner, which is why we focus on the case of known benefit functions and unknown costs. Additionally, we will demonstrate (\cf\ Sections \ref{sec:experiments} \& Appendix \ref{sec:multivariate-cost}) that our methodology correctly estimates $T_{\theta}(z)$ even if $B_{\theta}(z)$ is misspecified, (though the cost will not be estimated correctly). 

We call $T_{\theta}$ the \emph{agents' response map}.  Due to the agents response to $\theta$, the distribution of attributes has now shifted to $Q_{\theta} = {(T_{\theta})}_\# P$, which is the push-forward measure of $P$ through the agents' response $T_{\theta}$. We refer to this $Q_{\theta}$ as the \emph{performative distribution} or \emph{ex-post distribution} under the learner's action $\theta$. The map $\theta \mapsto Q_{\theta}$ is referred to as \emph{the performative distribution map} or simply \emph{distribution map}. 

The microfoundation model in eq. \eqref{eq:agent-util-maximization} covers a variety of performative prediction problems, that includes strategic classification \citep{hardt2016Strategic}, causal strategic classification \citep{miller2020Strategic, levanon2021Strategica, shavit2020Causal, somerstep2023Learning}, content creation \citep{hron2022Modeling}, and micro-economic models of labor markets \citep{coate1993Will}.

\subsection{Importance of microfoundation inference for downstream tasks}

Next, we cover two important use cases for estimating the agents microfoundations.

{\bf Performative risk minimization} As a response to the distribution shift, the learner aims to learn a predictor that minimizes the \emph{performative risk} which accounts for the agents' response:
\vspace{-.1cm}
\begin{equation}\textstyle \label{eq:performative-prediction}
\begin{aligned}
   \min\nolimits_{\theta \in \cF} \text{PR}(\theta), ~ 
   \text{PR}(\theta) \triangleq \Ex_{Z \sim Q_{\theta}} \big[ \ell(\theta; Z) \big]\,,
\end{aligned}
\end{equation} 
where $\ell$ is the appropriate prediction loss. Previous methods for performative risk minimization 
incorporate an agents' microfoundation to anticipate the distribution shift \citep{hardt2016Strategic,levanon2021Strategica, somerstep2023Learning}. As we demonstrate in the following example, estimates obtained using a misspecified microfoundation can be arbitrarily poor.



\begin{example}[Causal Strategic OLS \citep{shavit2020Causal}]
\label{ex:strategic-OLS}

In a linear regression setting, we highlight the importance of estimating the cost. Consider an \emph{ex-ante} 
distribution $P$ of agents' profiles, where a typical sample point $(X, Y) \in \reals^d \times \reals$
satisfies the following: $X \sim \bN(0, \bI_d)$
and for an $\eps\sim \bN(0, \sigma^2), ~ \eps \ind X$
the $Y = X^\top \theta^\star  + \eps$, for a $\theta^\star \in \reals^d$. 
Once a $\theta \in \reals^d$ is published, agents adjust their $X$'s to maximize their benefit 
$B_\theta(X) \triangleq \theta^\top X$, while the $Y$'s remain unchanged. Assuming that the true cost is 
$c(x, x') = \frac{1}{2}(x - x')^\top M (x - x')$, 
where $M$ is a predetermined positive definite matrix, the true agent's response map and performative risk are 
\begin{equation} \label{eq:example1-agents-map}
\begin{aligned}
T_\theta(X) &\textstyle= \argmax_{x'} {\theta}^\top x' - \frac{1}{2}(X - x')^\top M (X - x') \\
&= X + M^{-1}\theta\,, 
\end{aligned}
\end{equation} 
\begin{equation} \label{eq:example1-performative-risk}
\begin{aligned}
    \textstyle \textsc{PR} (\theta) &= \Ex\big[\{Y - \theta^\top T_\theta(X)\}^2\big]  \\
    &= \sigma^2 + \|\theta - \theta^\star\|_2^2 + \{\theta ^\top M^{-1} \theta\}^2\,. 
\end{aligned}
\end{equation} If the agents cost is misspecified to $\widehat c(x, x') = \frac{1}{2}\|x - x'\|_2^2$, then by letting $M = \bbI_d$ in the equations \eqref{eq:example1-agents-map} and \eqref{eq:example1-performative-risk} the learner anticipates the agents map and the corresponding performative risk as
\begin{equation}
\begin{aligned}
\widehat T_\theta(X) = X + \theta, ~ 
\widehat{\textsc{PR}}(\theta) = \sigma^2 + \|\theta - \theta^\star\|_2^2 + \|\theta\|_2^4\,.
\end{aligned}
\end{equation}
As we show in Lemma \ref{lemma:missp-map-minimization} that the learner who minimizes the wrongly anticipated $\widehat{\textsc{PR}}(\theta)$ decides on its optimizer $\hat \theta = c \theta^\star$ where $c$ is the only positive solution to equation $2 \|\theta^\star\|_2^2 c^3 + c - 1 = 0$. However,  the true performance risk for such a $\hat \theta$ is suboptimal, since
\[
\begin{aligned}
\textsc{PR} &(\hat \theta) - \min\nolimits_{\theta} \textsc{PR} (\theta)  \ge \textsc{PR} (\hat \theta)  - \textsc{PR} (0) \\
&= \{(c - 1)^2 - 1\} \|\theta^\star\|_2^2 + c^4 \{{\theta^\star}^\top M^{-1} \theta^\star\}^2
\end{aligned}
\] can be arbitrarily large depending on $M$, and, even trivially predicting every response as $\hat Y = 0^\top X = 0$ has a better performative risk than $\hat \theta ^\top X$.   To mitigate this suboptimality, one must estimate the cost. 
\end{example}

\textbf{Enforcing long-term fairness:}
Beyond accuracy, the call for practitioners to deploy socially responsible models in performative settings has grown; models are generally considered socially responsible if they either induce improvement (encourage increase of agent outcomes $Y$) or treat agents fairly post distribution shift \citep{miller2020Strategic, FairnessIsNotStatic_2020, estornell2021unfairness, liu2018delayed, 10.1145/3630106.3658929}. In either case, if one wishes to meet some definition of responsibility, inference on the individuals' response to models (e.g the cost for an individual to improve $Y$ in response to $\theta$) is a necessary step.

\vspace{-.2cm}
\section{Learning agents' response}
\vspace{-.2cm}

\label{sec:cost-estimation}

The first sticking point in agent cost estimation is a question identifiability: what class of costs can be learned if one only observes a finite number of samples over a finite number of model deployments? Note that estimating a general bivariate cost function $c(z, z')$ is conceptually impossible when we have access to samples from only finitely many \emph{ex-post} distributions $\{Q_{\theta_i}, i = 1, \dots, m\}$, in addition to the \emph{ex-ante} distribution $P$; even if we assume that we know the exact distributions. From the equality $Q_{\theta_i} = {(T_{\theta_i})}_\# P$, we can only know the finite number of actions $\{T_{\theta_i}(z); i = 1, \dots, m\}$ by the agent with the \emph{ex-ante} attribute $z$ when exposed to the models $\theta_1, \dots, \theta_m$. Thus, from the first-order condition of \eqref{eq:agent-util-maximization}, which is 
$$
\partial B_{\theta}(T_{\theta}(z)) - \partial_2 c(z, T_{\theta}(z)) = 0
$$ 
one can only evaluate the partial derivative of the cost at the points $\{(z, T_{\theta_i}(z)), i = 1, \dots, m\}$ and for a given $z$ it is impossible to evaluate $z' \mapsto \partial_2 c(z, z')$ beyond these points. 

\subsection{Agents' cost as Bregman Divergence}

In light of the non-identifiability of the generic cost $c(z,z')$ from a finite number of ex-ante distributions, we estimate the agents' cost within a class of Bregman divergences \citep{bregman1967relaxation}:
\[
c_{\varphi}(z,z') \triangleq \varphi(z') - \varphi(z) - \nabla\varphi(z)^\top(z'-z),
\]
where $\varphi:\reals^d\to\reals$ is a strictly convex Bregman potential. This assumption mitigates the problem of non-identifiability of the cost $c$ for finitely many deployed models $\theta$, because
estimating $c(z,z')$ is reduced to estimating $\varphi(z)$, thus eliminating the joint dependency of $(z,z')$.
In addition to making it possible to estimate the cost function, the Bregman divergence specification allows us to set the problem of cost estimation in an optimal transport framework. Furthermore, the class of Bregman divergences is broad and often subsumes classes of costs that are studied in other works. For example: \citealp[Section 5.2]{perdomo2020Performative}, \citealp[Section 5.1]{hardt2016Strategic} considers a simple $\|z-z'\|_2^2$ cost, which is a special case of Bregman divergence with potential $\|z\|_2^2$; in \cite{liu2023contextual} the cost is $c(z,z')=\tfrac{1}{2}(z-z')^{\top}A(z-z')$ for some positive definite matrix $A$, which is a special case of Bregman divergence with potential $\varphi(z) = \tfrac{1}{2}z^{\top}Az$; Other parametric restrictions imposed on distribution maps in the strategic learning literature include \cite{izzo2021How,miller2021Outside,jagadeesan2022Regret}. 

As the divergence is induced by its potential $\varphi$, from now on we focus on estimating $\varphi$. 
We consider two scenarios depending on whether a sample from the \emph{ex-ante} distribution $P$ is available or not.

\subsection{Learning from \emph{ex-ante} and \emph{ex-post} distributions}
\label{sec:scenerio1}
First, we assume that the learner has random samples $\cD_0 \triangleq\{z_{0, i} \}_{i = 1}^{n_0} \overset{\iid}{\sim}P$ from the \emph{ex ante} distribution and $\cD_{\theta_k} \triangleq\{z_{k, i} \}_{i = 1}^{n_k} \overset{\iid}{\sim} Q_{\theta_k} $ from $m$ \emph{ex-post} distributions $Q_{\theta_1}, \dots, Q_{\theta_m}$ related to the agents' response to $m$ prediction models $\theta_1,\dots,\theta_m$. Let $\widehat P$ and ${\widehat Q_{\theta_k}}$'s be the empirical counterparts of $ P$ and $Q_{\theta_k}$'s. We estimate $\varphi$ by solving
\begin{equation}
    \label{eq:procedure1}
    \begin{aligned}
        &\underset{\varphi \in \Phi_\cvx}{\argmin} \textstyle  \underset{\mu \in \Delta (\cZ)}{\min} \mathscr{L}(\mu,\widehat{\varphi},\widehat P, \widehat Q_{\theta_1}, \dots ,\widehat Q_{\theta_m})\\
&\text{ with }\mathscr{L}(\mu,\varphi, P,  Q_{\theta_1}, \dots , Q_{\theta_m})\triangleq \\
&
\textstyle W_2^2\big(\mu,  (\nabla\varphi)_\# P \big) + \sum_{k = 1}^m W_2^2 \big(\mu, (\nabla\varphi -\nabla B_{\theta_k})_\#Q_{\theta_k}\big), 
    \end{aligned}
\end{equation}
$\Phi_{\cvx}$ a class of convex potential functions, $W_2$ the 2-Wasserstein distance and $ T _\# P$ denotes the push-forward measure of $P$ through the map $T$. 

The problem \eqref{eq:procedure1} is motivated by a first-order characterization of the micro-level responses in \eqref{eq:agent-util-maximization}; for a generic $\theta$: 
\begin{equation}
    \begin{aligned}
\textstyle T_{\theta}(z) =  \argmax_{z' \in \cZ}\left\{\begin{aligned}B_{\theta}(z') - c_{\varphi^\star}(z,z')\end{aligned}\right\}\,, \\
\text{or, }\nabla \varphi^\star \circ T_{\theta}(z) - \nabla B_{\theta}\circ T_{\theta} (z) = \nabla \varphi^\star (z)\,,
\end{aligned}\label{eq:first-order-cond}
\end{equation}
where $\varphi^\star$ is the (unknown) Bregman potential for the cost. We recall that $(T_{\theta}) _\# P = Q_{\theta}$, which implies the distributions $(\nabla \varphi^\star - \nabla B_{\theta}) _\# Q_{\theta}$ and $ (\nabla \varphi^\star)_\# P$ must be identical:
\[
\begin{aligned}
& \{(\nabla \varphi^\star - \nabla B_{\theta}) \circ T_{\theta} \}_\# P \stackrel{\text{d}}{=} (\nabla \varphi^\star)_\# P ~~ \text{(from \eqref{eq:first-order-cond})} \\
&\implies   (\nabla \varphi^\star - \nabla B_{\theta}) _\# Q_{\theta} \stackrel{\text{d}}{=} (\nabla \varphi^\star)_{\#} P
\end{aligned}
\] 
where $\stackrel{\text{d}}{=}$ denotes equality in distributions.
It follows that for $\theta_1\dots \theta_m$, $(\nabla \varphi^\star)_\# P$, $(\nabla \varphi^\star - \nabla B_{\theta_1}) _\# Q_{\theta_1},$ $ \dots, (\nabla \varphi^\star - \nabla B_{\theta_m}) _\# Q_{\theta_m}$, are all identically distributed. In optimization \eqref{eq:procedure1} we estimate $\varphi^{\star}$ by exploiting this equality in the distribution, where we minimize the 2-Wasserstein variance. We use the (2-)Wasserstein distance in \eqref{eq:procedure1} because agent responses may not have common support, and it also allows us to leverage algorithms for fast computation of Wasserstein barycenters to solve \eqref{eq:procedure1} in a block coordinate descent algorithm (\cf\ see \citet{10.1023/A:1017501703105} for details on convergence). We summarize our method in Algorithm \ref{alg:pushforward-alignment}. 

\begin{algorithm}
\caption{Aligning agent responses}\label{alg:pushforward-alignment}
\begin{algorithmic}[1]
\State {\bfseries Input:} initial estimate $\widehat{\varphi}$ and $\widehat P, \widehat Q_{\theta_1}, \dots \widehat Q_{\theta_m}$
\Repeat
\State Compute the barycenter: $$
\textstyle \widehat{\mu}\gets \argmin_{\mu}  \textstyle \mathscr{L}(\mu,\widehat{\varphi},\widehat P, \widehat Q_{\theta_1}, \dots ,\widehat Q_{\theta_m})
$$

\State Update the Bregman potential:
\begin{equation}\label{solve:varphi}
\textstyle \widehat{\varphi}\gets \argmin_{\varphi \in \Phi_{\text{cvx}}}  \textstyle \mathscr{L}(\widehat{\mu},\varphi,\widehat P, \widehat Q_{\theta_1}, \dots ,\widehat Q_{\theta_m})
\end{equation}
\Until{converged}
\end{algorithmic}
\end{algorithm}

Although it is often desirable to estimate the $\varphi$ nonparametrically, for a higher dimensional $Z = (X, Y)$ one may further restrict $\varphi$ to a suitable parametric class of potential functions and find the optimal $\widehat\varphi$ within that class. We provide an example of such a parametric specification in Section \ref{sec:parametric}, where we assume that the cost function is quadratic, \ie\ $c(z, z') = \frac12 (z - z')^\top \bM (z - z')$ for a positive definite $\bM$, in which case the potentials are parameterized as $\varphi(z) = \frac{1}{2}z^\top \bM z$. 

\vspace{-0.2cm}
\subsection{Learning from \emph{ex-post} distributions}

Assume that only observations of the \emph{ex post} (performative) distributions $Q_{\theta_k}$'s are available. 
Similarly to before, the first-order optimality conditions of the \eqref{eq:agent-util-maximization} show that the push-forwards of the agent responses $Q_{\theta_k}$ under $\nabla B_{\theta_k} - \nabla\varphi$, $k\in[m]$ are identical, which suggests we estimate $\varphi$  aligning the pushforward distributions of the agent responses under the $(\nabla\varphi - \nabla B_{\theta_k})$'s:
\vspace{-0.1cm}
\begin{align} \label{eq:procedure2}
\textstyle
&\widehat \varphi \in \underset{\varphi \in \Phi_\cvx}{\argmin} \underset{\mu\in \Delta(\cZ)}{\min }\tilde{\mathscr{L}}(\mu,\varphi, \widehat Q_{\theta_1}, \dots , \widehat Q_{\theta_m}) \,\\
&\text{ where }\tilde{\mathscr{L}}(\mu,\varphi, Q_{\theta_1}, \dots , Q_{\theta_m})\text{ is defined as } \nonumber\\
&=\textstyle\sum_{k = 1}^m W_2^2 \big(\mu, (\nabla\varphi - \nabla B_{\theta_k})_\#Q_{\theta_k} \big)\nonumber\,.
\end{align}

\vspace{-.4cm}
\section{Theoretical analysis}
\label{sec:theory}

\paragraph{Identifiability of the cost function: }
A crucial step in analyzing the estimated cost is to verify whether it is identifiable from the proposed algorithms. More specifically, we investigate whether the potential functions obtained from \eqref{eq:procedure1} and \eqref{eq:procedure2} induce a Bergman divergence identical to $c_{\varphi^\star}$ (the true cost function), when the empirical distributions $\widehat P, \widehat Q_{\theta}, \dots, \widehat Q_{\theta_m}$ are substituted by their respective population versions $P, Q_{\theta}, \dots, Q_{\theta_m}$.
Since a Bergman divergence remains identical under any linear adjustment to its inducing potential, we verify whether $\Phi^\star \triangleq \{ \varphi^\star (z) + \alpha_0 + \alpha_1 ^\top z \}$ contains
\begin{equation} \label{eq:identifiability1}
\begin{cases}
   \begin{aligned}
       &  \textstyle  \underset{\varphi \in \Phi_\cvx}{\argmin} \underset{\mu \in \Delta (\cZ)}{\min}  \mathscr{L}(\mu,\varphi, P, Q_{\theta_1}, \dots , Q_{\theta_m})
   \end{aligned}  & \textit{(ex-ante)} \\
   \begin{aligned}
      &  \textstyle  \underset{\varphi \in \Phi_\cvx}{\argmin} \underset{\mu\in \Delta(\cZ)}{\min } \tilde{\mathscr{L}}(\mu,\varphi, Q_{\theta_1}, \dots , Q_{\theta_m})
   \end{aligned}
     & \textit{(ex-post)}
\end{cases}  
\end{equation}
To verify this, we first observe that the objectives of \eqref{eq:identifiability1} are identical to zero when $\varphi = \varphi^\star$, as we argue after \eqref{eq:procedure1} that $(\nabla \varphi^\star)_\# P$, $(\nabla \varphi^\star - \nabla B_{\theta}) _\# Q_{\theta},$ $ \dots, (\nabla \varphi^\star - \nabla B_{\theta_m}) _\# Q_{\theta_m}$ are identically distributed. Since these objectives are non-negative, it is clear that they must be zero at their minimum value, which further implies that for any minimizer $\tilde \varphi$, all $m+1$ distributions used to calculate the barycenter distances must be \emph{identical to the barycenters}. In the following lemma, we make use of this observation to derive an equivalence condition to verify whether the cost is identifiable. 

\begin{theorem} \label{lemma:identifiable}
    In the ex-ante problem, denote the transportation maps from the barycenter to the measures $ P,  Q_{\theta}, \dots  Q_{\theta_m}$ as $T_0, T_1, \dots, T_m$, and similarly, in (ex-post), denote them as $T_1, \dots, T_m$. Then an equivalent condition for \eqref{eq:identifiability1} is that all conservative solutions $h$ (\ie\ can be written as the derivative of a convex function) of the following equations are constant functions.
    \vspace{-.1cm}
    \begin{equation}
        \Bigg \{\begin{aligned}
            \textstyle  h \circ T_0 (z)  = \dots  = h \circ T_m (z) ~ \text{for all } z & ~~\text{(ex-ante)}\\
           h \circ T_1 (z) = \dots  = h \circ T_m (z) ~ \text{for all } z &~~ \text{(ex-post)}
        \end{aligned}
        \vspace{-.2cm}
    \end{equation}
    
\end{theorem}
In practice, this condition can be used to verify whether the cost is uniquely identified from the available dataset. As an example, consider the following corollary.
\begin{corollary} \label{lemma:identifiable2}
    With access to the \emph{ex-ante} distribution $P$ assume that only one performative distribution $Q_{\theta_1}$ is observed. If $B_{\theta_1}$ is strictly concave with a finite maximizer, then the condition in Theorem \ref{lemma:identifiable} is satisfied. Thus, the cost function is identified only from these two distributions.
\end{corollary}


\paragraph{Rate of convergence for cost estimation:} In addition to identifiability, we also establish a convergence rate for our proposed estimator using two distributions.
Specifically, we estimate $\nabla \varphi$ within a parametric class $\{\psi_\gamma: \gamma \in \Gamma\}$ using samples from two distributions: an ex-ante ($P$) and a single ex-post ($Q_{\theta}$). In this case, the loss utilized to estimate $\gamma$ in Algorithm $\ref{alg:pushforward-alignment}$ is
\begin{equation}
    \textstyle\bL(\Pi, \gamma) \triangleq \int \Pi(dz, dz') \| \psi_\gamma(z) - \psi_\gamma(z') + \nabla B_{\theta}(z')\|_2^2\,,
\end{equation} 
and the population value ($\gamma^\star$) and the estimated value ($\widehat \gamma$) of the parameter can be calculated as
\begin{align*}
&\textstyle  \gamma^\star \triangleq \argmin_\gamma \min_{\Pi \in \Delta(P, Q_{\theta})} \bL(\Pi, \gamma),\\
&\textstyle \widehat \gamma \triangleq \argmin_\gamma \min_{\Pi \in \Delta(\widehat P ,\widehat  Q_{\theta})} \bL(\Pi, \gamma)\,,
\end{align*} 
where $\widehat P$ and $\widehat Q_{\theta}$ are the empirical counterparts of $P$ and $Q_{\theta}$, and $\Delta(\mu,\nu)$ is the set of probability measures on $\cZ \times \cZ$ with marginals $\mu$ and $\nu$. In the following theorem, we establish an $\ell_2$ convergence rate between $\widehat \gamma$ and $\gamma^\star$ in the case where the OT loss used for cost estimation satisfies a strong convexity assumption. Note that this strong convexity assumption implies that we are also assuming that the cost is identifiable.
\begin{theorem} \label{thm:rate-of-convergence}
    Suppose that $P$ and $Q_{\theta}$ are compactly supported on $\mathbb{R}^d$ and are absolutely continuous with respect to the Lebesgue measure with bounded second moments. Denote $\widehat{P}$ and $\widehat{Q}_{\theta}$ as their empirical counterparts, each calculated using $n$ $\iid$  samples.  Furthermore, assume that both $\gamma \mapsto \min_{\Pi \in \Delta(P, Q_{\theta})} \bL(\Pi, \gamma)$ and $\gamma \mapsto \min_{\Pi \in \Delta(\widehat P,\widehat  Q_{\theta})} \bL(\Pi, \gamma)$ are strongly convex. Then there exists a $K > 0$ such that 
    \vspace{-0.2cm}
    \begin{equation} \label{eq:upper-bound}
        \textstyle \Ex\big [\|\widehat \gamma - \gamma^\star \|_2^2\big] \le  K n ^{- \nicefrac{2}{d}}\,.
    \end{equation}
\end{theorem} 
This rate was established using convergence results for the empirical Wasserstein distance from \citet[Theorem 2]{manole2024sharp}. Extending the upper bound in \eqref{eq:upper-bound} to multiple ex-post distributions $Q_{\theta_1}, \dots, Q_{\theta_m}$ or to nonidentical sample sizes for ex-ante and ex-post distributions is possible, but, to the best of our knowledge, similar convergence results for the empirical Wasserstein distance that we may require are not available in the literature and we leave this for future work.

\paragraph{Impact of estimating the cost.}
Having estimated the $\widehat \varphi$, we can plug it in eq. \eqref{eq:first-order-cond} and estimate the response map from the equation
\begin{equation}
    \nabla \widehat \varphi(\widehat T_\theta(z)) - \nabla B_\theta(\widehat T_\theta (x)) = \nabla \widehat\varphi(x)\,,
\end{equation} which is then used in the performative loss function to estimate the performative optimal (PO) solution:
\begin{equation}\label{eq:plugin}
    \widehat \theta = \argmin_\theta \Ex_{P_n}[\ell(\theta, \widehat T_\theta (Z))]\,. 
\end{equation} To understand the comparison of $\text{PR}(\widehat \theta)$ against the  $\min_\theta \text{PR}(\theta)$ we recall the analysis in \citet{lin2023plug}; especially Theorem 2. They decompose the risk-difference as
\begin{equation}\label{eq:plugin_error}
    \text{PR}(\widehat \theta) - \min_\theta \text{PR}(\theta) \le 2(\texttt{MisspErr} + \texttt{StatErr})\,. 
\end{equation} The first term is the error due to misspecification in our performative response map, often depending on the problem at hand and users' choice of the class of potential functions. The second term is a statistical error in estimating this response map. For a parametric specification of $\varphi$ we establish a rate of convergence $\cO_P(n^{-\nicefrac{1}{d}})$ for the said parameters (\cf\  Theorem \ref{thm:rate-of-convergence}). Under standard regularity conditions (\eg\ smoothness of the loss function and response map) one can understand that the statistical error will have the same $\cO_P(n^{-\nicefrac{2}{d}})$ rate of convergence. 

\section{Experiments}\label{sec:experiments}

In this section, we demonstrate the effectiveness, robustness, and competitiveness of our method in semi-synthetic setups. We use real data distributions as the ex-ante distributions and simulate agents' strategic behavior as utility-maximizing, consistent with \eqref{eq:agent-util-maximization}.

\textbf{Dataset}: We use the same credit scoring dataset from \citet{perdomo2020Performative}, which is publicly available on Kaggle \citep{cukierski2011credit}. The performative prediction framework is particularly relevant in the credit scoring context, as individuals may face restricted access to credit once a predictive model is deployed by a credit bureau. To counter these constraints, agents adapt their behavior to manipulate the system after a credit predictor is introduced. In this dataset, the target variable is a default indicator ($Y=\texttt{SeriousDlqin2yrs}$), equal to one if an individual is 90 days delinquent or worse, and zero otherwise. Explanatory variables include metrics such as the total balance on credit cards and personal lines of credit, divided by the total credit limits ($X=\texttt{RevolvingUtilizationOfUnsecuredLines}$).

\vspace{-.2cm}
\subsection{Estimating $\varphi'$ and $T_{\theta}$}
\vspace{-.2cm}

In this subsection, we demonstrate the performance of our method to estimate $\varphi'$ and $T_{\theta}$ using both ex-ante and ex-post distributions when the benefit function $B_{\theta}$ is known. To estimate $\varphi$ in the uni-dimensional case with one ex-ante distribution and one ex-post distribution, we use a non-parametric estimator for $\varphi'$. Specifically, $\varphi'$ is estimated via monotonic (isotonic) regression and subsequently integrated to estimate $\varphi$. Since $\texttt{RevolvingUtilizationOfUnsecuredLines}$ has the highest marginal correlation with the target $Y$ and is manipulable by agents, we focus on this feature. We first correctly specify $B_\theta(x)$ as $|\theta| \sqrt{x}$ and then consider two cases of misspecification for $B_\theta(x)$: $|\theta| \log{(x)}$ and $|\theta| \sqrt[3]{x}$. Results for additional types of misspecification are presented in Appendix \ref{sec:append_extra_results}. The predictor $\theta$ is obtained by fitting a logistic regression that predicts $Y$ from $X$. This analysis helps us assess the robustness of our method in the presence of misspecification in the benefit function.

We begin by evaluating the accuracy of our method in estimating $\varphi'$ when $n=200$ samples from each of the two distributions are available. From Figure \ref{fig:exp1_varphi_est}, we observe that the estimated functions are close to the true one when the benefit $B_\theta$ is correctly specified. However, the estimates exhibit significant bias when the benefit is misspecified. Despite this, Figure \ref{fig:exp1_T_est} shows that this bias does not carry over to the estimates of $T_\theta$. Intuitively, this occurs because the misspecification of the benefit function is absorbed during the estimation of $\varphi'$ (thereby compromising its accuracy) but does not affect the overall predictive performance of the method. In Figure \ref{fig:exp1_T_est}, we estimate $T_{c\theta}$ for different values of $c$ and assess whether this robustness persists even when attempting to estimate the map $T_{\tilde{\theta}}$ for values of $\tilde{\theta}$ different from those used to induce the ex-post distribution and estimate $\varphi'$, \ie, $\theta$. An interesting pattern observed in Figure \ref{fig:exp1_T_est} is that the estimates deteriorate as $c$ increases in $c\theta$; this occurs because samples are transported over greater distances to regions of the space where data points are sparse, leading to noisier estimates. For both figures, we report estimates based on different random seeds.
\begin{figure}[h]
\centering
\includegraphics[width=.49\textwidth]{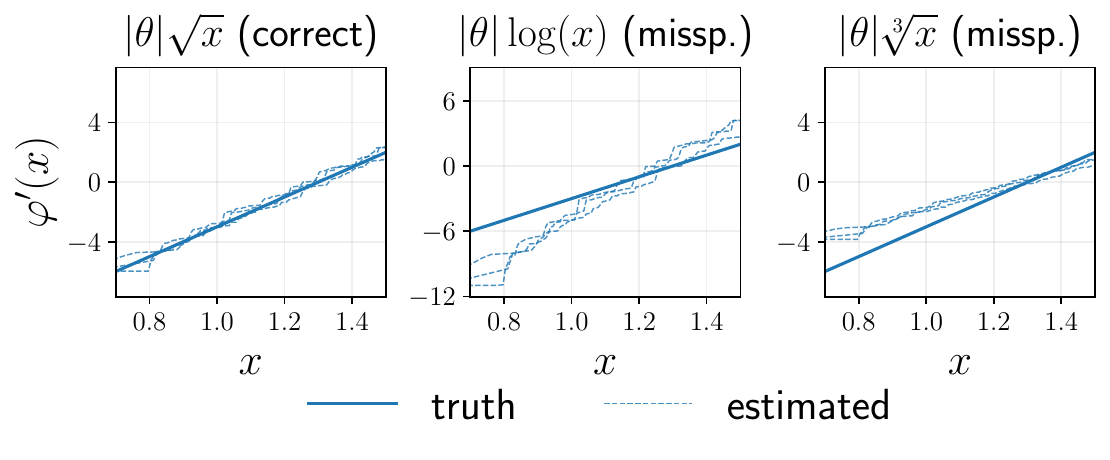}
\vspace{-0.2cm}
\caption{\small The function $\varphi'$ is well estimated when he benefit function $B_\theta $ is correctly specified. On the other hand, misspecification of $B_\theta $ leads to biased estimates of $\varphi'$.}
\label{fig:exp1_varphi_est}
\end{figure}

\begin{figure}[h]
\centering
\includegraphics[width=.5\textwidth]{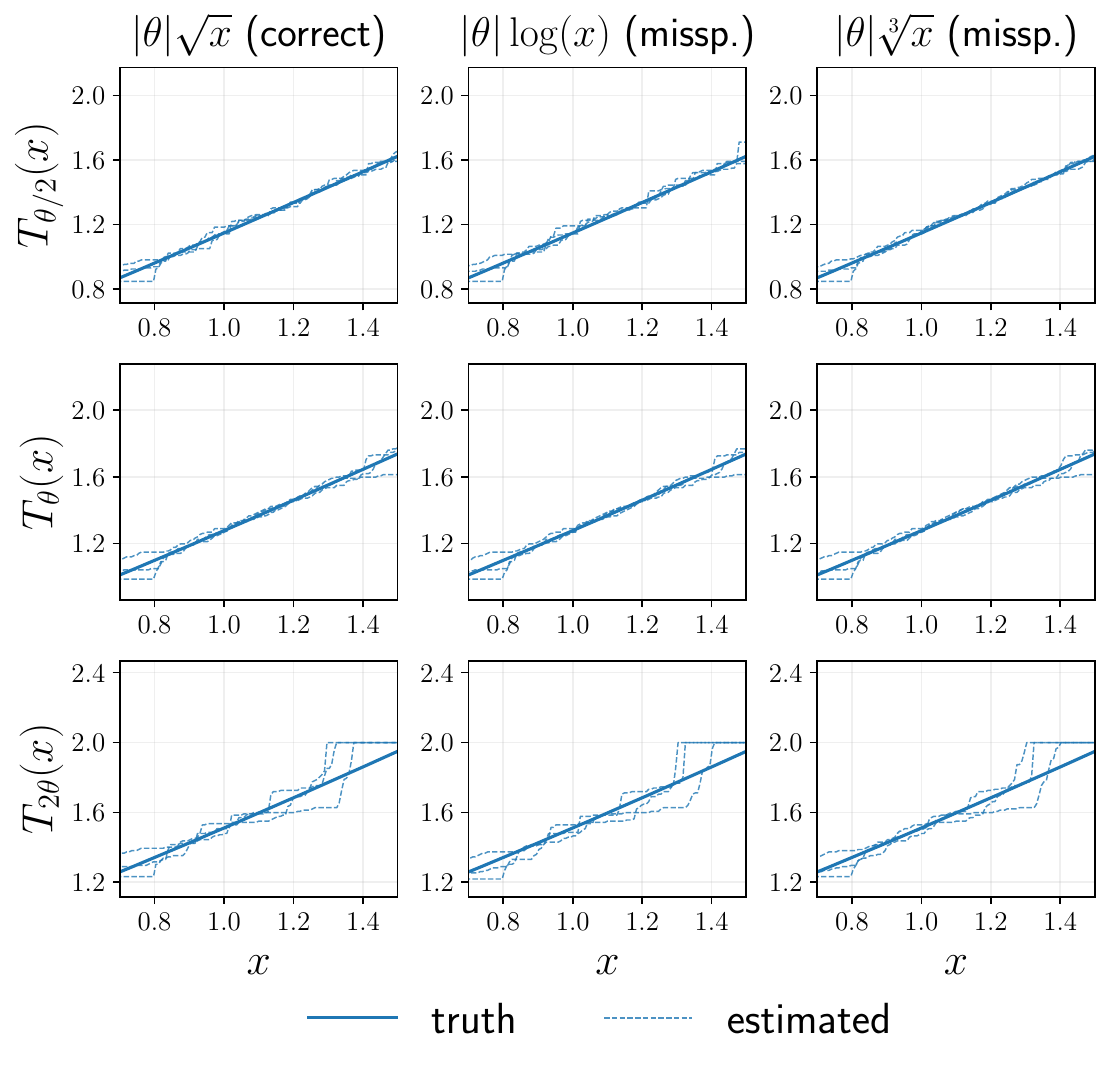}
\vspace{-0.2cm}
\caption{\small The estimation of the transport map $T_{\tilde{\theta}}$ is robust to the misspecification of the benefit function for values of $\tilde{\theta}$ different from those used to induce the ex-post distribution and estimate $\varphi'$, \ie, $\theta$.}
\label{fig:exp1_T_est}
\end{figure}

\begin{figure*}[t]
\centering
\includegraphics[width=.95\textwidth]{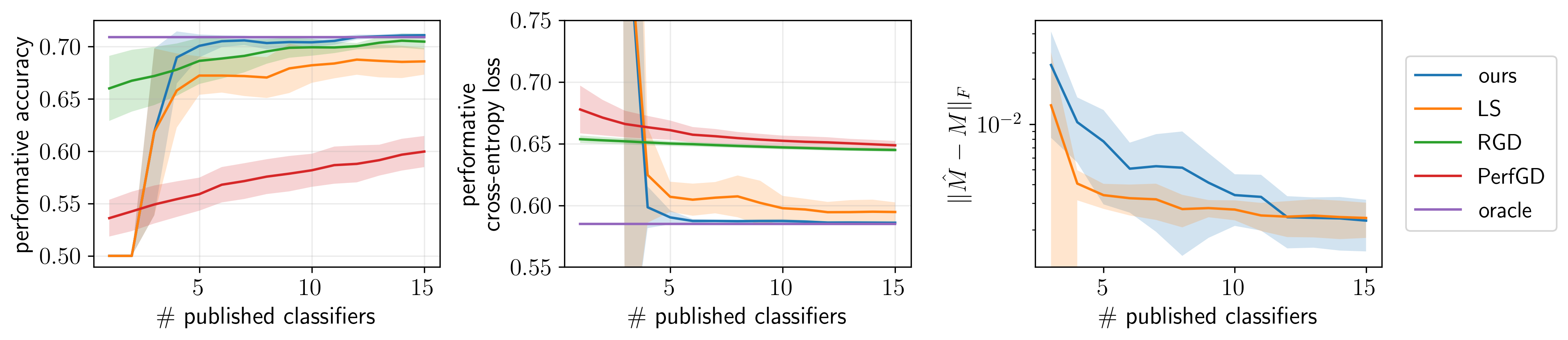}
\vspace{-0.2cm}
\caption{Performative performance for different $\#$ published classifiers: the plots depict the performative test accuracy/cross-entropy loss as the $\#$ classifiers increases. Compared to the baselines, our method converges much faster to the optimal classifier.}
\label{fig:exp2}
\vspace{-0.2cm}
\end{figure*}

Finally, we show how the estimation error for $T_\theta$ behaves with different sample sizes when the $B_\theta$ is correctly specified. We assume that the number of data points in each one of the distributions is $n\in\{10,25,50,100,200\}$ and consider the estimation error to be $\int_\cX |T_\theta(x)-\hat{T}_\theta(x)|dx$ where $\cX$ denotes the support of our observations, which is fixed through sample sizes. Figure \ref{fig:exp1_ss_est} shows that the estimation error of our method decreases as the sample size grows. The dots give the average error, and the error bars the standard deviation across random seeds.
\begin{figure}[h]
\centering
\includegraphics[width=.35\textwidth]{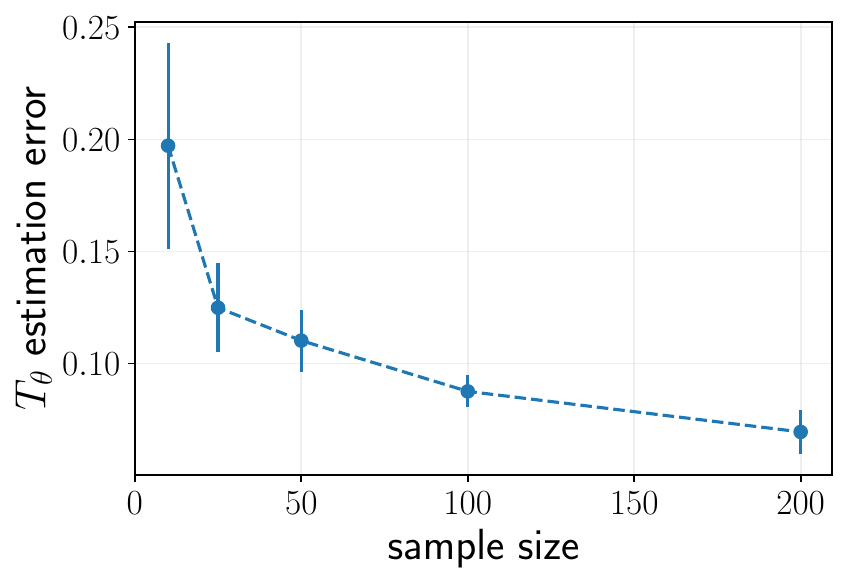}
\vspace{-0.2cm}
\caption{\small The estimation error of our method, when estimating $T_\theta$, decreases with growing sample size.}
\vspace{-0.5cm}
\label{fig:exp1_ss_est}
\end{figure}

\subsection{Minimizing the performative risk}
\label{sec:parametric}

In this subsection, we apply our method to minimize the performative risk. We work in a multidimensional setup in which $X$ is composed of \texttt{RevolvingUtilizationOfUnsecuredLines}, \texttt{DebtRatio} (monthly debt payments divided by monthly gross income), \texttt{Monthly income}, and a constant representing the intercept. To simplify things, we consider all the features in $X$ to be performative and assume a linear benefit with a quadratic cost structure as in Example \ref{ex:strategic-OLS}. To obtain ex-ante and ex-post data we publish a set of classifiers $\{\theta_0,\cdots,\theta_{K-1}\}$ randomly sampled from a Gaussian distribution centered around the optimal classifier in the ex-ante distribution $P$. We then estimate $M$ using Algorithm \ref{alg:pushforward-alignment} where the minimization step over the set of all matrices is done using the Cholesky decomposition $M=LL^\top$ by optimizing for $L$ over the set of lower triangular matrices. Once the map $T_\theta$ is estimated, we obtain a plugin estimator \eqref{eq:plugin} for a logistic regression by minimizing the cross-entropy loss. More details are in Appendix \ref{sec:detail_append}.

We compare our method with four baselines: (1) \textit{oracle}, the classifier that uses the true performative map to minimize performative cross-entropy loss in the test dataset; (2) Repeated Gradient Descent (\texttt{RGD}) \citep{perdomo2020Performative}; (3) Performative Gradient Descent (\texttt{PerfGD}) \citep{izzo2022How}; and (4) the two-step method proposed by \citet{miller2021Outside} which includes least-squares estimator for the matrix $M$, also used by \citet[Algorithm 2]{jagadeesan2022Regret}. We call the last method \texttt{LS}. For \texttt{RGD} and \texttt{PerfGD}, we ran experiments with different learning rates and selected those maximizing performative accuracy after the last classifier was published.  For \texttt{PerfGD}, we assume that the induced distributions are Gaussian with a known mean structure (as it depends on $T_\theta$), a known covariance, and a known derivative of its mean with respect to\footnote{In equation 4 of \citet{izzo2022How}.} $\theta$. 

The left and middle plots in Figure \ref{fig:exp2} present the test set accuracies and cross-entropy losses across various classification strategies as the number of published classifiers increases. The right plot illustrates a comparison between our estimator for $M$ and the least squares approach from \citet{miller2021Outside}. In all cases, our method demonstrates competitive performance relative to the baselines. Two additional observations are noteworthy. First, despite our implementations of \texttt{RGD} and \texttt{PerfGD} being optimized for those methods, their performance can be suboptimal in several scenarios, as shown in Figure \ref{fig:exp2} and Appendix \ref{sec:append_extra_results}. The Gaussianity assumption inherent in \texttt{PerfGD}, which is a standard assumption in \citet{izzo2022How}, appears to contribute to its weaker performance and failure to converge to the oracle in some cases. Second, since our model is correctly specified in this scenario, we observe that its performance converges to that of the oracle as the sample size grows, \ie, the number of published classifiers, which is consistent with the behavior predicted by \eqref{eq:plugin_error}.


\vspace{-0.2cm}
\section{Summary and discussion}\label{sec:discussion}

We propose algorithms to infer the performative distribution shift map in settings where samples are from strategic agents. Each agent responds to a model $\theta$ by solving $\argmax_{z'}B_{\theta}(z') - c_{\varphi}(z', z)$, where $B_{\theta}(\cdot)$ is a known benefit function and $c_{\varphi}(\cdot, \cdot)$ is an unknown cost induced by the Bregman potential $\varphi$, whose structure offers analytical convenience without sacrificing much generality. Our key method estimates the gradient of the Bregman potential $\nabla \varphi$ by aligning ex-ante and ex-post distributions using optimal transport. Our experiments show that the method is effective and robust to misspecifications in the benefit function $B_{\theta}(\cdot)$.

Performative prediction remains highly relevant in ML research and is applicable in a broad range of areas. Accurate knowledge of performative microfoundations is crucial for the deployment of effective and socially responsible models. Our method of learning agent response maps can make such deployment more accessible in strategic prediction problems, allowing practitioners to use fast-converging algorithms that require these response maps without fear of misspecification. While our paper focuses on prediction problems with strategic agents, learning distribution shifts in broader performative prediction problems, potentially involving non-strategic agents, is an important question for future research and is beyond the scope of this paper.

\section{Code Availability}
\label{sec:code}

The code used in this work is available at \url{https://github.com/felipemaiapolo/microfoundation_inference} \citep{maiapolo2025code}.

\bibliography{main}

\begin{thebibliography}{47}
\providecommand{\natexlab}[1]{#1}
\providecommand{\url}[1]{\texttt{#1}}
\expandafter\ifx\csname urlstyle\endcsname\relax
  \providecommand{\doi}[1]{doi: #1}\else
  \providecommand{\doi}{doi: \begingroup \urlstyle{rm}\Url}\fi

\bibitem[Bregman(1967)]{bregman1967relaxation}
L.~M. Bregman.
\newblock The relaxation method of finding the common point of convex sets and its application to the solution of problems in convex programming.
\newblock \emph{USSR Computational Mathematics and Mathematical Physics}, 7\penalty0 (3):\penalty0 200--217, January 1967.
\newblock ISSN 0041-5553.
\newblock \doi{10.1016/0041-5553(67)90040-7}.

\bibitem[Campbell(1979)]{campbell1979assessing}
Donald~T Campbell.
\newblock Assessing the impact of planned social change.
\newblock \emph{Evaluation and program planning}, 2\penalty0 (1):\penalty0 67--90, 1979.

\bibitem[Chen et~al.(2020)Chen, Liu, and Podimata]{chen2020Learning}
Yiling Chen, Yang Liu, and Chara Podimata.
\newblock Learning strategy-aware linear classifiers.
\newblock In \emph{Proceedings of the 34th {{International Conference}} on {{Neural Information Processing Systems}}}, {{NIPS}}'20, pages 15265--15276, {Red Hook, NY, USA}, December 2020. {Curran Associates Inc.}
\newblock ISBN 978-1-71382-954-6.

\bibitem[Coate and Loury(1993)]{coate1993Will}
Stephen Coate and Glenn~C. Loury.
\newblock Will {{Affirmative-Action Policies Eliminate Negative Stereotypes}}?
\newblock \emph{The American Economic Review}, 83\penalty0 (5):\penalty0 1220--1240, 1993.
\newblock ISSN 0002-8282.

\bibitem[Cukierski(2011)]{cukierski2011credit}
Will Cukierski.
\newblock Give me some credit.
\newblock \url{https://kaggle.com/competitions/GiveMeSomeCredit}, 2011.

\bibitem[Cutler et~al.(2021)Cutler, Drusvyatskiy, and Harchaoui]{cutler2021Stochastic}
Joshua Cutler, Dmitriy Drusvyatskiy, and Zaid Harchaoui.
\newblock Stochastic optimization under time drift: Iterate averaging, step-decay schedules, and high probability guarantees.
\newblock In \emph{Thirty-{{Fifth Conference}} on {{Neural Information Processing Systems}}}, May 2021.

\bibitem[D'Amour et~al.(2020)D'Amour, Srinivasan, Atwood, Baljekar, Sculley, and Halpern]{FairnessIsNotStatic_2020}
Alexander D'Amour, Hansa Srinivasan, James Atwood, Pallavi Baljekar, D.~Sculley, and Yoni Halpern.
\newblock Fairness is not static: deeper understanding of long term fairness via simulation studies.
\newblock In \emph{Proceedings of the 2020 Conference on Fairness, Accountability, and Transparency}, FAT* '20, page 525–534, New York, NY, USA, 2020. Association for Computing Machinery.
\newblock ISBN 9781450369367.
\newblock \doi{10.1145/3351095.3372878}.
\newblock URL \url{https://doi.org/10.1145/3351095.3372878}.

\bibitem[Dong et~al.(2018)Dong, Roth, Schutzman, Waggoner, and Wu]{dong2018Strategic}
Jinshuo Dong, Aaron Roth, Zachary Schutzman, Bo~Waggoner, and Zhiwei~Steven Wu.
\newblock Strategic {{Classification}} from {{Revealed Preferences}}.
\newblock In \emph{Proceedings of the 2018 {{ACM Conference}} on {{Economics}} and {{Computation}}}, {{EC}} '18, pages 55--70, {New York, NY, USA}, June 2018. {Association for Computing Machinery}.
\newblock ISBN 978-1-4503-5829-3.
\newblock \doi{10.1145/3219166.3219193}.

\bibitem[Drusvyatskiy and Xiao(2023)]{drusvyatskiy2023Stochastic}
Dmitriy Drusvyatskiy and Lin Xiao.
\newblock Stochastic {{Optimization}} with {{Decision-Dependent Distributions}}.
\newblock \emph{Mathematics of Operations Research}, 48\penalty0 (2):\penalty0 954--998, May 2023.
\newblock ISSN 0364-765X.
\newblock \doi{10.1287/moor.2022.1287}.

\bibitem[Eilat et~al.(2023)Eilat, Finkelshtein, Baskin, and Rosenfeld]{eilat2023strategic}
Itay Eilat, Ben Finkelshtein, Chaim Baskin, and Nir Rosenfeld.
\newblock Strategic classification with graph neural networks, 2023.

\bibitem[Estornell et~al.(2021)Estornell, Das, Liu, and Vorobeychik]{estornell2021unfairness}
Andrew Estornell, Sanmay Das, Yang Liu, and Yevgeniy Vorobeychik.
\newblock Unfairness despite awareness: Group-fair classification with strategic agents, 2021.

\bibitem[Fang and Moro(2011)]{fang2011Chapter}
Hanming Fang and Andrea Moro.
\newblock Chapter 5 - {{Theories}} of {{Statistical Discrimination}} and {{Affirmative Action}}: {{A Survey}}.
\newblock In Jess Benhabib, Alberto Bisin, and Matthew~O. Jackson, editors, \emph{Handbook of {{Social Economics}}}, volume~1, pages 133--200. {North-Holland}, January 2011.
\newblock \doi{10.1016/B978-0-444-53187-2.00005-X}.

\bibitem[Ghalme et~al.(2021)Ghalme, Nair, Eilat, {Talgam-Cohen}, and Rosenfeld]{ghalme2021Strategica}
Ganesh Ghalme, Vineet Nair, Itay Eilat, Inbal {Talgam-Cohen}, and Nir Rosenfeld.
\newblock Strategic {{Classification}} in the {{Dark}}.
\newblock In \emph{Proceedings of the 38th {{International Conference}} on {{Machine Learning}}}, pages 3672--3681. {PMLR}, July 2021.

\bibitem[Hardt et~al.(2016)Hardt, Megiddo, Papadimitriou, and Wootters]{hardt2016Strategic}
Moritz Hardt, Nimrod Megiddo, Christos Papadimitriou, and Mary Wootters.
\newblock Strategic {{Classification}}.
\newblock In \emph{Proceedings of the 2016 {{ACM Conference}} on {{Innovations}} in {{Theoretical Computer Science}}}, {{ITCS}} '16, pages 111--122, {New York, NY, USA}, January 2016. {Association for Computing Machinery}.
\newblock ISBN 978-1-4503-4057-1.
\newblock \doi{10.1145/2840728.2840730}.

\bibitem[Harris et~al.(2022)Harris, Ngo, Stapleton, Heidari, and Wu]{harris2022Strategic}
Keegan Harris, Dung Daniel~T. Ngo, Logan Stapleton, Hoda Heidari, and Steven Wu.
\newblock Strategic {{Instrumental Variable Regression}}: {{Recovering Causal Relationships From Strategic Responses}}.
\newblock In \emph{Proceedings of the 39th {{International Conference}} on {{Machine Learning}}}, pages 8502--8522. {PMLR}, June 2022.

\bibitem[Horowitz and Rosenfeld(2023)]{horowitz2023Causal}
Guy Horowitz and Nir Rosenfeld.
\newblock Causal {{Strategic Classification}}: {{A Tale}} of {{Two Shifts}}, February 2023.

\bibitem[Hron et~al.(2022)Hron, Krauth, Jordan, Kilbertus, and Dean]{hron2022Modeling}
Jiri Hron, Karl Krauth, Michael Jordan, Niki Kilbertus, and Sarah Dean.
\newblock Modeling content creator incentives on algorithm-curated platforms.
\newblock In \emph{The {{Eleventh International Conference}} on {{Learning Representations}}}, September 2022.

\bibitem[Hron et~al.(2023)Hron, Krauth, Jordan, Kilbertus, and Dean]{hron2023Modeling}
Jiri Hron, Karl Krauth, Michael Jordan, Niki Kilbertus, and Sarah Dean.
\newblock Modeling content creator incentives on algorithm-curated platforms.
\newblock In \emph{The {{Eleventh International Conference}} on {{Learning Representations}}}, February 2023.

\bibitem[Izzo et~al.(2021)Izzo, Ying, and Zou]{izzo2021How}
Zachary Izzo, Lexing Ying, and James Zou.
\newblock How to {{Learn}} when {{Data Reacts}} to {{Your Model}}: {{Performative Gradient Descent}}.
\newblock In \emph{Proceedings of the 38th {{International Conference}} on {{Machine Learning}}}, pages 4641--4650. {PMLR}, July 2021.

\bibitem[Izzo et~al.(2022)Izzo, Zou, and Ying]{izzo2022How}
Zachary Izzo, James Zou, and Lexing Ying.
\newblock How to {{Learn}} when {{Data Gradually Reacts}} to {{Your Model}}.
\newblock In \emph{Proceedings of {{The}} 25th {{International Conference}} on {{Artificial Intelligence}} and {{Statistics}}}, pages 3998--4035. {PMLR}, May 2022.

\bibitem[Jagadeesan et~al.(2022)Jagadeesan, Zrnic, and {Mendler-D{\"u}nner}]{jagadeesan2022Regret}
Meena Jagadeesan, Tijana Zrnic, and Celestine {Mendler-D{\"u}nner}.
\newblock Regret {{Minimization}} with {{Performative Feedback}}.
\newblock In \emph{Proceedings of the 39th {{International Conference}} on {{Machine Learning}}}, pages 9760--9785. {PMLR}, June 2022.

\bibitem[Jagadeesan et~al.(2023)Jagadeesan, Garg, and Steinhardt]{jagadeesan2023SupplySidea}
Meena Jagadeesan, Nikhil Garg, and Jacob Steinhardt.
\newblock Supply-{{Side Equilibria}} in {{Recommender Systems}}.
\newblock In \emph{Thirty-Seventh {{Conference}} on {{Neural Information Processing Systems}}}, November 2023.

\bibitem[Jang et~al.(2022)Jang, Park, Lee, and Bastani]{jang2022Sequential}
Sooyong Jang, Sangdon Park, Insup Lee, and Osbert Bastani.
\newblock Sequential {{Covariate Shift Detection Using Classifier Two-Sample Tests}}.
\newblock In \emph{Proceedings of the 39th {{International Conference}} on {{Machine Learning}}}, pages 9845--9880. {PMLR}, June 2022.

\bibitem[Lechner et~al.(2023)Lechner, Urner, and Ben-David]{LechnerUnknown2023}
Tosca Lechner, Ruth Urner, and Shai Ben-David.
\newblock Strategic classification with unknown user manipulations.
\newblock In \emph{Proceedings of the 40th International Conference on Machine Learning}, ICML'23. JMLR.org, 2023.

\bibitem[Levanon and Rosenfeld(2021)]{levanon2021Strategica}
Sagi Levanon and Nir Rosenfeld.
\newblock Strategic {{Classification Made Practical}}.
\newblock In \emph{Proceedings of the 38th {{International Conference}} on {{Machine Learning}}}, pages 6243--6253. {PMLR}, July 2021.

\bibitem[Levanon and Rosenfeld(2022)]{levanon2022Generalized}
Sagi Levanon and Nir Rosenfeld.
\newblock Generalized {{Strategic Classification}} and the {{Case}} of {{Aligned Incentives}}.
\newblock In \emph{Proceedings of the 39th {{International Conference}} on {{Machine Learning}}}, pages 12593--12618. {PMLR}, June 2022.

\bibitem[Lin and Zrnic(2023)]{lin2023plug}
Licong Lin and Tijana Zrnic.
\newblock Plug-in performative optimization.
\newblock \emph{arXiv preprint arXiv:2305.18728}, 2023.

\bibitem[Liu et~al.(2018)Liu, Dean, Rolf, Simchowitz, and Hardt]{liu2018delayed}
Lydia~T. Liu, Sarah Dean, Esther Rolf, Max Simchowitz, and Moritz Hardt.
\newblock Delayed {{Impact}} of {{Fair Machine Learning}}.
\newblock \emph{arXiv:1803.04383 [cs, stat]}, March 2018.

\bibitem[Liu et~al.(2023)Liu, Yang, Wang, and Sun]{liu2023contextual}
Pangpang Liu, Zhuoran Yang, Zhaoran Wang, and Will~Wei Sun.
\newblock Contextual dynamic pricing with strategic buyers.
\newblock \emph{arXiv preprint arXiv:2307.04055}, 2023.

\bibitem[Maia~Polo(2024)]{maiapolo2025code}
F.~Maia~Polo.
\newblock microfoundation inference code, 2024.
\newblock URL \url{https://github.com/felipemaiapolo/microfoundation_inference}.

\bibitem[Maity et~al.(2021)Maity, Mukherjee, Yurochkin, and Sun]{maity2021Does}
Subha Maity, Debarghya Mukherjee, Mikhail Yurochkin, and Yuekai Sun.
\newblock Does enforcing fairness mitigate biases caused by subpopulation shift?
\newblock In \emph{Advances in {{Neural Information Processing Systems}}}, November 2021.

\bibitem[Manole and Niles-Weed(2024)]{manole2024sharp}
Tudor Manole and Jonathan Niles-Weed.
\newblock Sharp convergence rates for empirical optimal transport with smooth costs.
\newblock \emph{The Annals of Applied Probability}, 34\penalty0 (1B):\penalty0 1108--1135, 2024.

\bibitem[{Mendler-D{\"u}nner} et~al.(2020){Mendler-D{\"u}nner}, Perdomo, Zrnic, and Hardt]{mendler-dunner2020Stochastic}
Celestine {Mendler-D{\"u}nner}, Juan~C. Perdomo, Tijana Zrnic, and Moritz Hardt.
\newblock Stochastic optimization for performative prediction.
\newblock In \emph{Proceedings of the 34th {{International Conference}} on {{Neural Information Processing Systems}}}, {{NIPS}}'20, pages 4929--4939, {Red Hook, NY, USA}, December 2020. {Curran Associates Inc.}
\newblock ISBN 978-1-71382-954-6.

\bibitem[{Mendler-D{\"u}nner} et~al.(2022){Mendler-D{\"u}nner}, Ding, and Wang]{mendler-dunner2022Anticipating}
Celestine {Mendler-D{\"u}nner}, Frances Ding, and Yixin Wang.
\newblock Anticipating {{Performativity}} by {{Predicting}} from {{Predictions}}.
\newblock In \emph{Advances in {{Neural Information Processing Systems}}}, October 2022.

\bibitem[Miller et~al.(2020)Miller, Milli, and Hardt]{miller2020Strategic}
John Miller, Smitha Milli, and Moritz Hardt.
\newblock Strategic {{Classification}} is {{Causal Modeling}} in {{Disguise}}.
\newblock \emph{arXiv:1910.10362 [cs, stat]}, February 2020.

\bibitem[Miller et~al.(2021)Miller, Perdomo, and Zrnic]{miller2021Outside}
John~P. Miller, Juan~C. Perdomo, and Tijana Zrnic.
\newblock Outside the {{Echo Chamber}}: {{Optimizing}} the {{Performative Risk}}.
\newblock In \emph{Proceedings of the 38th {{International Conference}} on {{Machine Learning}}}, pages 7710--7720. {PMLR}, July 2021.

\bibitem[Moro and Norman(2003)]{moro2003Affirmative}
Andrea Moro and Peter Norman.
\newblock Affirmative action in a competitive economy.
\newblock \emph{Journal of Public Economics}, 87\penalty0 (3-4):\penalty0 567--594, March 2003.
\newblock ISSN 00472727.
\newblock \doi{10.1016/S0047-2727(01)00121-9}.

\bibitem[Moro and Norman(2004)]{moro2004general}
Andrea Moro and Peter Norman.
\newblock A general equilibrium model of statistical discrimination.
\newblock \emph{Journal of Economic Theory}, 114\penalty0 (1):\penalty0 1--30, January 2004.
\newblock ISSN 00220531.
\newblock \doi{10.1016/S0022-0531(03)00165-0}.

\bibitem[Perdomo et~al.(2020)Perdomo, Zrnic, Mendler-D{\"u}nner, and Hardt]{perdomo2020Performative}
Juan Perdomo, Tijana Zrnic, Celestine Mendler-D{\"u}nner, and Moritz Hardt.
\newblock Performative prediction.
\newblock In \emph{International Conference on Machine Learning}, pages 7599--7609. PMLR, 2020.

\bibitem[Shavit et~al.(2020)Shavit, Edelman, and Axelrod]{shavit2020Causal}
Yonadav Shavit, Benjamin~L. Edelman, and Brian Axelrod.
\newblock Causal strategic linear regression.
\newblock In \emph{Proceedings of the 37th {{International Conference}} on {{Machine Learning}}}, volume 119 of \emph{{{ICML}}'20}, pages 8676--8686. {JMLR.org}, July 2020.

\bibitem[Somerstep et~al.(2023)Somerstep, Sun, and Ritov]{somerstep2023Learning}
Seamus Somerstep, Yuekai Sun, and Ya'acov Ritov.
\newblock Learning in reverse causal strategic environments with ramifications on two sided markets.
\newblock In \emph{{{NeurIPS}} 2023 {{Workshop}} on {{Algorithmic Fairness}} through the {{Lens}} of {{Time}} ({{AFT2023}})}, December 2023.

\bibitem[Somerstep et~al.(2024)Somerstep, Ritov, and Sun]{10.1145/3630106.3658929}
Seamus Somerstep, Ya'acov Ritov, and Yuekai Sun.
\newblock Algorithmic fairness in performative policy learning: Escaping the impossibility of group fairness.
\newblock In \emph{Proceedings of the 2024 ACM Conference on Fairness, Accountability, and Transparency}, FAccT '24, page 616–630, New York, NY, USA, 2024. Association for Computing Machinery.
\newblock ISBN 9798400704505.
\newblock \doi{10.1145/3630106.3658929}.
\newblock URL \url{https://doi.org/10.1145/3630106.3658929}.

\bibitem[Sundaram et~al.(2021)Sundaram, Vullikanti, Xu, and Yao]{sundaram2021paclearning}
Ravi Sundaram, Anil Vullikanti, Haifeng Xu, and Fan Yao.
\newblock Pac-learning for strategic classification, 2021.

\bibitem[Tseng(2001)]{10.1023/A:1017501703105}
P.~Tseng.
\newblock Convergence of a block coordinate descent method for nondifferentiable minimization.
\newblock \emph{J. Optim. Theory Appl.}, 109\penalty0 (3):\penalty0 475–494, June 2001.
\newblock ISSN 0022-3239.
\newblock \doi{10.1023/A:1017501703105}.
\newblock URL \url{https://doi.org/10.1023/A:1017501703105}.

\bibitem[Tsirtsis et~al.(2024)Tsirtsis, Tabibian, Khajehnejad, Singla, Schölkopf, and Gomez-Rodriguez]{Tsirtsis_2024}
Stratis Tsirtsis, Behzad Tabibian, Moein Khajehnejad, Adish Singla, Bernhard Schölkopf, and Manuel Gomez-Rodriguez.
\newblock Optimal decision making under strategic behavior.
\newblock \emph{Management Science}, 70\penalty0 (12):\penalty0 8506–8519, December 2024.
\newblock ISSN 1526-5501.
\newblock \doi{10.1287/mnsc.2021.02567}.
\newblock URL \url{http://dx.doi.org/10.1287/mnsc.2021.02567}.

\bibitem[Villani(2009)]{villani2009Optimal}
C{\'e}dric Villani.
\newblock \emph{Optimal Transport: Old and New}.
\newblock Number 338 in Grundlehren Der Mathematischen {{Wissenschaften}}. {Springer}, {Berlin}, 2009.
\newblock ISBN 978-3-540-71049-3.

\bibitem[Zrnic et~al.(2022)Zrnic, Mazumdar, Sastry, and Jordan]{zrnic2022Who}
Tijana Zrnic, Eric Mazumdar, S.~Shankar Sastry, and Michael~I. Jordan.
\newblock Who {{Leads}} and {{Who Follows}} in {{Strategic Classification}}?, January 2022.

\end{thebibliography}
\bibliographystyle{plainnat}

\newpage
\onecolumn
\section*{Checklist}

 \begin{enumerate}

 \item For all models and algorithms presented, check if you include:
 \begin{enumerate}
   \item A clear description of the mathematical setting, assumptions, algorithm, and/or model. [Yes]
   \item An analysis of the properties and complexity (time, space, sample size) of any algorithm. [Yes]
   \item (Optional) Anonymized source code, with specification of all dependencies, including external libraries. [Yes]
 \end{enumerate}

 \item For any theoretical claim, check if you include:
 \begin{enumerate}
   \item Statements of the full set of assumptions of all theoretical results. [Yes]
   \item Complete proofs of all theoretical results. [Yes]
   \item Clear explanations of any assumptions. [Yes]     
 \end{enumerate}

 \item For all figures and tables that present empirical results, check if you include:
 \begin{enumerate}
   \item The code, data, and instructions needed to reproduce the main experimental results (either in the supplemental material or as a URL). [Yes]
   \item All the training details (e.g., data splits, hyperparameters, how they were chosen). [Yes]
    \item A clear definition of the specific measure or statistics and error bars (e.g., with respect to the random seed after running experiments multiple times). [Yes]
    \item A description of the computing infrastructure used. (e.g., type of GPUs, internal cluster, or cloud provider). [Not Applicable]
 \end{enumerate}

 \item If you are using existing assets (e.g., code, data, models) or curating/releasing new assets, check if you include:
 \begin{enumerate}
   \item Citations of the creator If your work uses existing assets. [Yes]
   \item The license information of the assets, if applicable. [Not Applicable]
   \item New assets either in the supplemental material or as a URL, if applicable. [Not Applicable]
   \item Information about consent from data providers/curators. [Not Applicable]
   \item Discussion of sensible content if applicable, e.g., personally identifiable information or offensive content. [Not Applicable]
 \end{enumerate}

 \item If you used crowdsourcing or conducted research with human subjects, check if you include:
 \begin{enumerate}
   \item The full text of instructions given to participants and screenshots. [Not Applicable]
   \item Descriptions of potential participant risks, with links to Institutional Review Board (IRB) approvals if applicable. [Not Applicable]
   \item The estimated hourly wage paid to participants and the total amount spent on participant compensation. [Not Applicable]
 \end{enumerate}
 \end{enumerate}

\newpage
\makeatletter
\renewcommand{\aistatstitle}[1]{
  \hsize\textwidth
  \linewidth\hsize
  \toptitlebar 
  {\centering
  {\Large\bfseries #1 \par}}
  \bottomtitlebar 
  \vskip 0.1in  
}
\makeatother

\newpage

\newpage
\appendix
%

%
\runningauthor{Bracale, Maia Polo, Maity, Somerstep, Banerjee, Sun}

\onecolumn
\aistatstitle{Microfoundation Inference for Strategic Prediction: \\
Supplementary Materials}
\section{Supplementary proofs}

\label{sec:proof}

\begin{lemma}
    \label{lemma:missp-map-minimization}
    The misspecified performative risk $\widehat{\textsc{PR}}(\theta) = \sigma^2 + \|\theta - \theta^\star\|_2^2 + \|\theta\|_2^4$  in Example \ref{ex:strategic-OLS} is minimized at $\hat \theta = c \theta^\star$, where $c$ is the only positive solution to equation $2 \|\theta^\star\|_2^2 c^3 + c - 1 = 0$. 
\end{lemma}

\begin{proof}[Proof of Lemma \ref{lemma:missp-map-minimization}]
   The  misspecified performative risk and its first-order for its minimization are
\begin{equation}
\begin{aligned}
     & \textstyle
    \widehat{\text{PR}}(\theta) = \sigma^2 + \|\theta - \theta^\star\|_2^2 + \|\theta\|_2^4 \\
    & \textstyle \nabla_\theta \widehat{\text{PR}}(\theta) = 2 (\theta - \theta^\star) + 4 \|\theta\|_2^2 \theta = 0 \\
    \implies & \textstyle \big\{1 + 2\|\theta\|_2^2\big\} \theta = \theta^\star ~  \text{or,} ~   \textstyle \theta = c \theta^\star\,,
\end{aligned}
\end{equation} for some $0 <c \le 1$. It remains to optimize over $c \in (0, 1]$. Letting $\theta = c\theta^\star$ in $\widehat{\text{PR}}(\theta)$ we obtain 
\begin{equation}
    \textstyle \widehat{\text{PR}}(c\theta^\star) = \sigma^2 + (c - 1)^2 \|\theta^\star\|_2^2 + c^4 \|\theta^\star\|_2^4 = \sigma^2 + \|\theta^\star\|_2^2 \big\{ (c - 1)^2 + c^4 \|\theta^\star\|_2^2 \big\}
\end{equation} and at its minimum the following first-order condition is satisfied: 
\begin{equation}
    \textstyle  \|\theta^\star\|_2^2 \big\{2(c - 1) + 4c^3 \|\theta^\star\|_2^2 \big\} = 0 ~ \text{or,} ~ 2\|\theta^\star\|_2^2 c^3 + c - 1= 0\,.
\end{equation} What remains to be seen is whether the above equation has a solution in the interval $(0, 1]$. This is evident from the following observation: for $f(c) \triangleq 2\|\theta^\star\|_2^2 c^3 + c - 1$ we see that 
\begin{equation}
    f(0) = -1 < 0, ~ f(1) = 2 \|\theta^\star\|_2 \ge 0, ~  \text{ and }f'(c) = 6 \|\theta^\star\|_2^2 c^2  + 1>  0 \,.
\end{equation}
Since $\theta$ is continuous and strictly increasing, we conclude that there is only one root of $f(c) = 0$. Furthermore, from $f(0) < 0 \le f(1)$ we conclude that the root is in the interval $(0, 1]$. 
  This proves the lemma.  
\end{proof}

\begin{proof}[Proof of Theorem \ref{lemma:identifiable}]
    We shall only prove the lemma for (ex-ante). The proof for (ex-post) is identical. As $T_0, T_1, \dots, T_m$ are maps from the barycenter measure to $ P,  Q_{\theta_1}, \dots  Q_{\theta_m}$ in \eqref{eq:identifiability1} and these measures are identical for the optimum $\varphi$, the barycenter problem must satisfy
    \begin{equation}
      \textstyle    W_2^2\big(\mu,  (\nabla\varphi)_\# {P} \big) +  \sum_{k = 1}^m W_2^2 \big(\mu, (\nabla\varphi - \nabla B_{\theta_k})_\#{Q}_{\theta_k} \big) = 0
    \end{equation} and thus each of the Wasserstein are zero. This leads to 
    \begin{equation} \label{eq:tech1}
        \textstyle
        z = (\nabla\varphi) \circ T_0(z) = (\nabla\varphi - \nabla B_{\theta_1}) \circ T_1 (z) = \dots = (\nabla\varphi - \nabla B_{\theta_m}) \circ T_m (z). 
    \end{equation}
    To realize this, let us focus on the first Wasserstein distance: $W_2^2\big(\mu,  (\nabla\varphi)_\# {P} \big) = 0$, which implies that $\nabla \varphi$ is a measure-preserving push-forward map from $P$ to $\mu$. Furthermore, $\nabla \varphi$ is conservative (a derivative of a convex function). Now, since we used the quadratic cost $\|\cdot\|_2^2$ in our optimal transport, it follows from Brenier's theorem that $\nabla \varphi$ is an optimal transport map from $P$ to $\mu$. Furthermore,  $T_0$ is the optimal transport map from $\mu$ to $P$. Thus, a combination of them is the identity map: $z = (\nabla \varphi ) \circ T_0(z)$. We can use similar arguments for $W_2^2 \big(\mu, (\nabla\varphi - \nabla B_{\theta_k})_\#{Q}_{\theta_k} \big) = 0$ to conclude: $z = (\nabla \varphi - \nabla B_{\theta_k} ) \circ T_k(z)$



    The equation \eqref{eq:tech1} reduces to the following equivalent condition that
    \begin{equation} \label{eq:tech2}
       \textstyle  (\nabla\varphi) \circ T_0(z) = (\nabla\varphi - \nabla B_{\theta_1}) \circ T_1 (z) = \dots = (\nabla\varphi - \nabla B_{\theta_m}) \circ T_m (z)
    \end{equation} The solutions of the above equations are unique up to a constant addition is equivalent to $\nabla \varphi$ being unique up to a constant addition. Thus, the cost is unique if and only if $\nabla \varphi \in \{\nabla \varphi^\star + \alpha\} \triangleq \cG$. We write the homogeneous equation to \eqref{eq:tech2} 
    \begin{equation}
        \label{eq:tech-homog}
        \textstyle h\circ T_0(z) = \textstyle h\circ T_1(z) = \dots = \textstyle h\circ T_m(z)
    \end{equation} and define the space of all its solutions as $\cH$. Note that for any $g \in \cG$, $h \in \cH$ and $t \in \reals$ we have $g + t h \in \cG$, because for any $k \in [m]$ we have
    \[
    \begin{aligned}
        \textstyle (g + th - \nabla B_{f_k}) \circ T_k(z)
        & =  \textstyle (g  - \nabla B_{f_k}) \circ T_k(z) + th  \circ T_k(z)\\
        & = \textstyle g  \circ T_0(z) + th  \circ T_0(z)\,.
    \end{aligned}
    \] Thus, $\cG = \{\nabla \varphi^\star(z) + \alpha, \alpha\in \reals^d\}$ if and only if all the functions in $\cH$ are constant.  
\end{proof}
\begin{proof}[Proof of Corollary \ref{lemma:identifiable2}]
Since there is only one $Q_{\theta_1}$ in addition to $P$, the optimization is reduced to 
\[
\textstyle  \underset{\varphi \in \Phi_\cvx}{\min} \underset{\mu \in \Delta (\cZ)}{\min} W_2^2\big(\mu,  (\nabla\varphi)_\# {P} \big) +   W_2^2 \big(\mu, (\nabla\varphi - \nabla B_{\theta_1})_\#{Q}_{\theta_1} \big)
\]
    In case of barycenter with two distributions, the optimization simplifies to 
\[
\textstyle \underset{\varphi \in \Phi_\cvx}{\min} ~   W_2^2 \big((\nabla\varphi)_\# {P}, (\nabla\varphi - \nabla B_{\theta_1})_\#{Q}_{\theta_1} \big)
\] At the optimum $\nabla \varphi$ the objective is zero and hence $(\nabla\varphi)_\# {P}$ and $(\nabla\varphi - \nabla B_{\theta_1})_\#{Q}_{\theta_1}$ are equal to the barycenter. Let $T_0$ and $T_1$ be the optimal transport map from the barycenter to the distributions $P$ to $Q_{\theta_1}$. Then, we need to verify that the equation 
\[
 g( T_0 (z) )= g(T_1 (z) ) \text{ or, equivalently } g(z ) = g(T_1 \circ T_0^{-1}(z))
\] are trivial. Here,  $T = T_1 \circ T_0^{-1}$ is the optimal transport map from $P$ to $Q_{\theta_1}$, which, by the first order condition of \eqref{eq:agent-util-maximization} satisfies
\begin{equation} \label{eq:eq1}
    \nabla \varphi \circ T (z) = \nabla \varphi (z) - \nabla B_{\theta_1}(z). 
\end{equation} 

We just need to prove that the only solutions for the equation 
 \begin{equation} \label{eq:eq2}
     g(T(z)) = g(z)
 \end{equation} is constant solutions. According to \eqref{eq:eq2} we have 
 \begin{equation} \label{eq:eq3}
     g(z) = g (T^n(z)), ~~ T^n(z) \triangleq \underbrace{T \circ \dots \circ T}_{n~ \text{times}}(z)
 \end{equation}

As \eqref{eq:eq1} is a first-order condition for $T(z) = \argmax_{z'} B_{\theta_1}(z') - d_{\varphi^\star}(z, z')$ we notice that $T$ is a proximal gradient ascend step for the strictly concave function $B_{\theta_1}$. Thus, for any $z$ we have $T^n(z) \to z^\star \triangleq  \argmax_{z'} B_{\theta_1}(z')$ as $n \to \infty$. This and \eqref{eq:eq3} implies 
\[
g(z) = g(z^\star)~~ \text{for all} ~~ z\,.
\] This implies $g$ is a constant function, which completes the proof.

\end{proof}

\begin{proof}[Proof of Theorem \ref{thm:rate-of-convergence}]

Let us denote the $\Pi^\star_\gamma$ and $\widehat\Pi_\gamma$ as the optimal couplings defined as:
\begin{equation}
    \Pi^\star_\gamma = \argmin_{\Pi \in \Delta(P , Q_{\theta})} \bL(\Pi, \gamma), ~~  \widehat\Pi_\gamma = \argmin_{\Pi \in \Delta(\widehat P ,\widehat  Q_{\theta})} \bL(\Pi, \gamma)
\end{equation} 
which are unique \citep{villani2009Optimal}. 
Using the local strong convexity of $\bL (\Pi^\star_\gamma, \gamma)$ at $\gamma^\star$ and local strong convexity of $\bL (\widehat \Pi_\gamma, \gamma)$ at $\widehat \gamma$ we have 
\begin{equation}
    \begin{aligned}
        &\textstyle \bL(\Pi^\star_{\widehat \gamma},\widehat \gamma) - \bL(\Pi^\star_{ \gamma^\star}, \gamma^\star) \ge \underbrace{\langle \widehat \gamma - \gamma^\star, \frac{d}{d\gamma}  \bL (\Pi^\star_{\gamma^\star}, \gamma^\star)\rangle}_{=0} + \frac{\kappa}{2} \|\widehat\gamma - \gamma^\star\|_2^2 \\
        &\textstyle \bL(\widehat \Pi_{ \gamma^\star}, \gamma^\star) - \bL(\widehat \Pi_{\widehat \gamma},\widehat \gamma)   \ge \underbrace{\langle \widehat \gamma - \gamma^\star, \frac{d}{d\gamma}  \bL (\widehat \Pi_{\widehat \gamma}, \widehat \gamma)\rangle}_{=0} + \frac{\kappa}{2} \|\widehat\gamma - \gamma^\star\|_2^2 \\
    \end{aligned}
\end{equation} We add the two equations to yield: 
\[
\begin{aligned}
    \kappa \|\widehat \gamma - \gamma^\star \|_2^2 & \le \bL(\Pi^\star_{\widehat \gamma},\widehat \gamma) - \bL(\Pi^\star_{ \gamma^\star}, \gamma^\star) + \bL(\widehat \Pi_{ \gamma^\star}, \gamma^\star) - \bL(\widehat \Pi_{\widehat \gamma},\widehat \gamma) \\
    & = \big |\bL(\Pi^\star_{\widehat \gamma},\widehat \gamma) - \bL(\widehat \Pi_{\widehat \gamma},\widehat \gamma)\big |  + \big|\bL(\widehat \Pi_{ \gamma^\star}, \gamma^\star)  - \bL(\Pi^\star_{ \gamma^\star}, \gamma^\star) \big |
\end{aligned}
\] We take an expectation over the randomness in $\widehat P$ and $\widehat Q_{\theta}$ to yield:

\[
\begin{aligned}
    \kappa \Ex\big [\|\widehat \gamma - \gamma^\star \|_2^2\big] & \le \Ex\big[ |\bL(\Pi^\star_{\widehat \gamma},\widehat \gamma) - \bL(\widehat \Pi_{\widehat \gamma},\widehat \gamma)|\big ]  + \Ex\big[|\bL(\widehat \Pi_{ \gamma^\star}, \gamma^\star)  - \bL(\Pi^\star_{ \gamma^\star}, \gamma^\star) |\big ] \\
    & \le \textstyle 2 \sup_{\gamma \in B_{\delta}(\gamma^{\star})} \Ex\big[ |\bL(\Pi^\star_{ \gamma}, \gamma) - \bL(\widehat \Pi_{ \gamma}, \gamma)|\big]\,.
\end{aligned}
\]
where $B_{\delta}(\gamma^{\star}) \triangleq \{\gamma: \|\gamma - \gamma^\star\|_2 \le \delta\}$
for some $\delta>0$. Thus, the $\ell_2$ error for estimating the parameter is
\[
\textstyle \Ex\big [\|\widehat \gamma - \gamma^\star \|_2^2\big] \le \frac2 \kappa \sup_{\gamma \in B_{\delta}(\gamma^{\star})} \Ex\big[ |\bL(\Pi^\star_{ \gamma}, \gamma) - \bL(\widehat \Pi_{ \gamma}, \gamma)|\big]\,.
\]

Under the assumption that $P$ is compactly supported on $\cX_P$ the superdomain $\widetilde \cX_P$ defined below is also compactly supported:
\[
\widetilde \cX_P = \{ \psi_\gamma (x): x \in \cX_P, \gamma \in B_{\delta}(\gamma^{\star})\}.
\] Since $P$ is compactly supperted, so it $Q_{\theta}$, as it follows the microfoundation. Defining $\cY$ as a support for $Q_{\theta}$, since $\cY$ the superdomain $\widetilde \cY$ is also compactly supported. 
\[
\widetilde \cY = \{\psi_\gamma(y) - b_{\theta}(y): y \in \cY , \gamma \in B_{\delta}(\gamma^{\star})\}
\] Defining $\tilde \psi_\gamma(y) = \psi_\gamma(y) - b_{\theta}(y)$, $\mu_\gamma  = (\psi_\gamma)_{\#} P \in \Delta(\widetilde \cX_P)$ and $\nu_\gamma = (\tilde \psi_\gamma)_{\#}  Q_{\theta} \in \Delta(\widetilde \cY)$ we notice 
\[
\textstyle \bL(\Pi^\star_{ \gamma}, \gamma) = W_2^2 \big (\mu_\gamma, \nu_\gamma  \big ), ~~ \bL(\widehat \Pi_{ \gamma}, \gamma) = W_2^2 \big (\widehat \mu_\gamma, \widehat\nu_\gamma   \big )
\] where $\widehat \mu_\gamma$ and $\widehat \nu_\gamma$ are empirical measures of $\mu$ and $\nu$:
\[
\textstyle \widehat \mu_\gamma  = (\psi_\gamma)_{\#} \widehat P, ~~ \widehat \nu_\gamma = (\tilde \psi_\gamma)_{\#} \widehat Q_{\theta}\,.
\] Since both the supports $\widetilde \cX_P$ and $\widetilde \cY$ are compact, we use the bound in \citet[Theorem 2]{manole2024sharp} (with $\alpha = 2$) to obtain:
\[
\textstyle \Ex\big [\|\widehat \gamma - \gamma^\star \|_2^2\big] \le \frac2 \kappa \sup_{\gamma \in B_{\delta}(\gamma^{\star})} \Ex\big[ |\bL(\Pi^\star_{ \gamma}, \gamma) - \bL(\widehat \Pi_{ \gamma}, \gamma)|\big] \le \frac2 \kappa C n ^{- \frac{2}{d}}
\] for some $C>0$. This completes the proof. 
    
\end{proof}

\section{Multi-dimensional cost estimation with convex neural network} \label{sec:multivariate-cost}

Here, we work in the multi-dimensional case with an ex-ante distribution and one ex-post distribution.  The $\nabla \varphi$ is estimated using a Single Layer Neural Network (SLNN). More specifically, we model the potential function $\varphi$ as a single hidden layer convex neural network: 
\[ \textstyle
\varphi_\gamma(x) = \sum_{j = 1}^{h} \delta_j \sigma(\omega_j^\top x+ \kappa_j), ~~ \gamma = (\Omega,  \kappa, \delta), \Omega = \{\omega_j;j \in [h]\} \in \reals^{h \times d}, \kappa \in \reals^h, \delta_j \ge 0\,.
\] where $\sigma$ is a strictly convex and increasing function. For this section, we use $\sigma(t) = \log(1 + e^t)$. 

\begin{figure}[H]
\vspace{-0.4cm}
    \centering
    \includegraphics[width=0.35\textwidth]{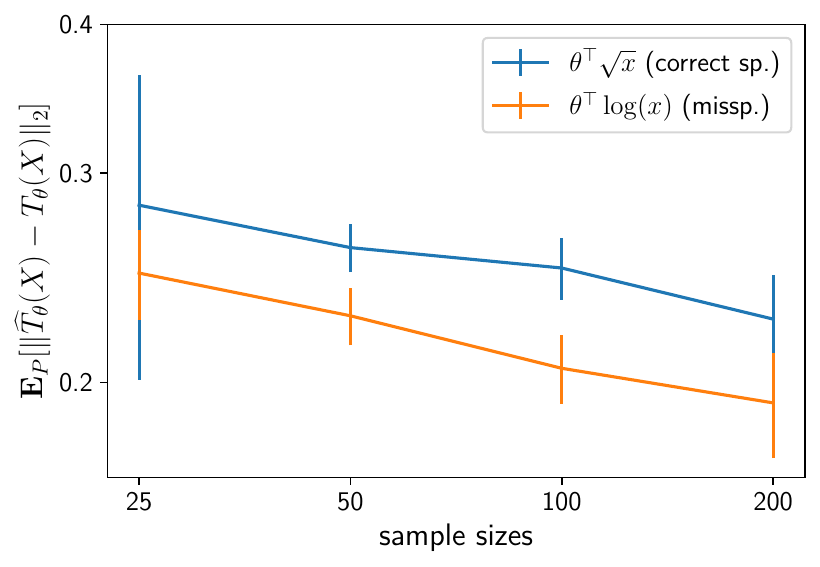}
    \caption{Estimation error for the map $T_\theta$ in the multivariate case with a convex neural network.}
    \label{fig:exp3}
    \vspace{-0.2cm}
\end{figure}

For our experiment, we consider $x\in \reals^5$ and five hidden units for the $\varphi_\gamma$. We use $\sqrt{\cdot}$ as the benefit function. With more details, we consider the benefit function the benefit $B_\theta(x') = \sum_j\theta_j x_j^{\frac12}$, where $\theta$ is a randomly generated vector with positive entries. We use 8 different random seeds and plot the average result. The unknown Bregman potential is $\varphi(x)=\frac{\sigma}{2}\|x\|^2$, implying $\nabla \varphi(x)= \sigma x$ (we assume $\sigma=.1$). The agents start with $X\sim P =  \bN(0, \bI_d)$, and then change their values of $X$ by maximizing their net gain:
\[
T_{\theta} (X) = \argmax_{x'} B_\theta(x') - c_{\varphi}(X,x').
\]
where in this case $c_{\varphi}(X,x') = \nicefrac{\sigma}{2}\|X-x'\|^2$. After we train the SLNN we get an estimate $\widehat{\varphi}$ and we estimate $T_{\theta}$ by 
\[
\widehat{T}_{\theta} (X) = \argmax_{x'} B_\theta(x') - c_{\widehat{\varphi}}(X,x').
\]

In Figure \ref{fig:exp3}, we compare the error in the estimation of $T_\theta$ for two benefit functions: $\sum_j\theta_j x_j^{\frac12}$, which is well specified, and $\sum_j\theta_j \log(x_j)$, which is misspecified. In both cases, we see an improved estimation for $T_\theta$ for a larger sample size, which again exhibits the robustness in estimating $T_\theta$ despite having a misspecified benefit.  

\section{Experiments}

\subsection{Model training procedure in Section \ref{sec:parametric}}\label{sec:detail_append}
In this experiment, agents can manipulate their features once a logistic regression classifier, $\theta=(\alpha, \beta)$, is published. In this context, $\alpha$ is the intercept term while $\beta$ is the slope. We assume\footnote{In our experiment, $\beta$ assumes negative entries.} $B_\theta(x)=|\beta|^\top x=-\beta^\top x$ \citep{perdomo2020Performative} and consider the case of parametric cost estimation with $\varphi(x)=\frac{1}{2}x^\top Mx$ where $M$ is a positive definite matrix to be learned. As we show in \eqref{eq:example1-agents-map}, agents' response in this setting is $T_\theta(x) = X - M^{-1}\beta$. For the main experiment, we assume $M=0.1* I_3$ and include more results in Appendix \ref{sec:append_extra_results}.

To apply our method to the problem of minimizing the performative risk, we publish a set of classifiers $\{\theta_0,\cdots,\theta_{K-1}\}$ randomly sampled from a Gaussian distribution centered around the optimal classifier in the ex-ante distribution $P$. We collect $n=250$ data points from the ex-ante distribution and each one of the ex-post distributions and then estimate $M$ using Algorithm \ref{alg:pushforward-alignment} where the minimization step over the set of all matrices is done using the Cholesky decomposition $M=LL^\top$ by optimizing for $L$ over the set of lower triangular matrices. After an estimate $\widehat{M}$ is obtained, we plug it in the response map to get $\widehat{T}_\theta(x) = X - \widehat{M}^{-1}\beta$. The (estimated) response map $\widehat{T}_\theta(x)$ is then used to learn the best performative logistic classifier by minimizing an estimate for the (unnormalized) performative cross-entropy loss
\begin{align*}
    \cL(\theta) &=   - \sum_{k=0}^{K-1} \sum_{i=0}^{n-1} \big\{Y_{ik}\log(p_{ik}) + (1-Y_{ik})\log(1-p_{ik})\big\}-\sum_{i=0}^{n-1} \big\{Y_{i}\log(p_{i}) + (1-Y_{i})\log(1-p_{i})\big\}.
     \vspace{-0.6cm}
\end{align*} 

The term $p_{ik}=\big\{1+\exp[-(\alpha+\beta^\top \widehat{T}_\theta(\widehat{X}_{ik}))]\big\}^{-1}$ with $\widehat{X}_{ik} = \widehat{T}^{-1}_{\theta_k}(X'_{ik})$, $X'_{ik}=T_{\theta_k}(X_{ik})$ accounts for samples from the \emph{ex-post} (with $k$ indexing inducement by $\theta_k$) distributions  and the term $p_{i}=\big\{1+\exp[-(\alpha+\beta^\top \widehat{T}_\theta(X_{i}))]\big\}$ accounts for samples from the ex-ante distribution. 

\subsection{Extra plots}\label{sec:append_extra_results}

In this section, we present some extra plots derived from what is presented in Section \ref{sec:experiments}. 

Figures \ref{fig:exp1_varphi_est_append} and \ref{fig:exp1_T_est_append} complement figures \ref{fig:exp1_varphi_est} and \ref{fig:exp1_T_est} from the main text with more types of misspecification

\begin{figure}[H]
\centering
\includegraphics[width=.95\textwidth]{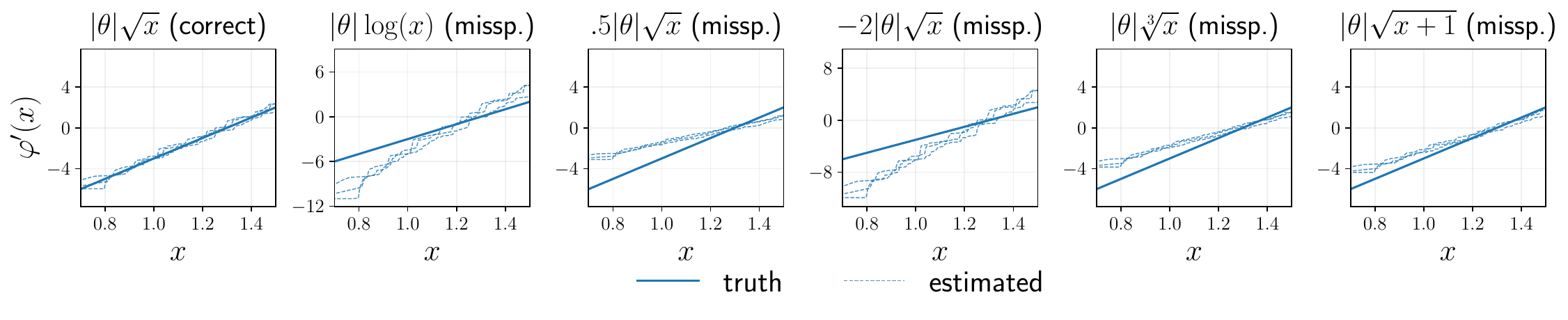}
\vspace{-0.2cm}
\caption{\small The function $\varphi'$ is well estimated when he benefit function $B_\theta $ is correctly specified. On the other hand, misspecification of $B_\theta $ leads to biased estimates of $\varphi'$.}
\label{fig:exp1_varphi_est_append}
\end{figure}

\begin{figure}[H]
\centering
\includegraphics[width=.95\textwidth]{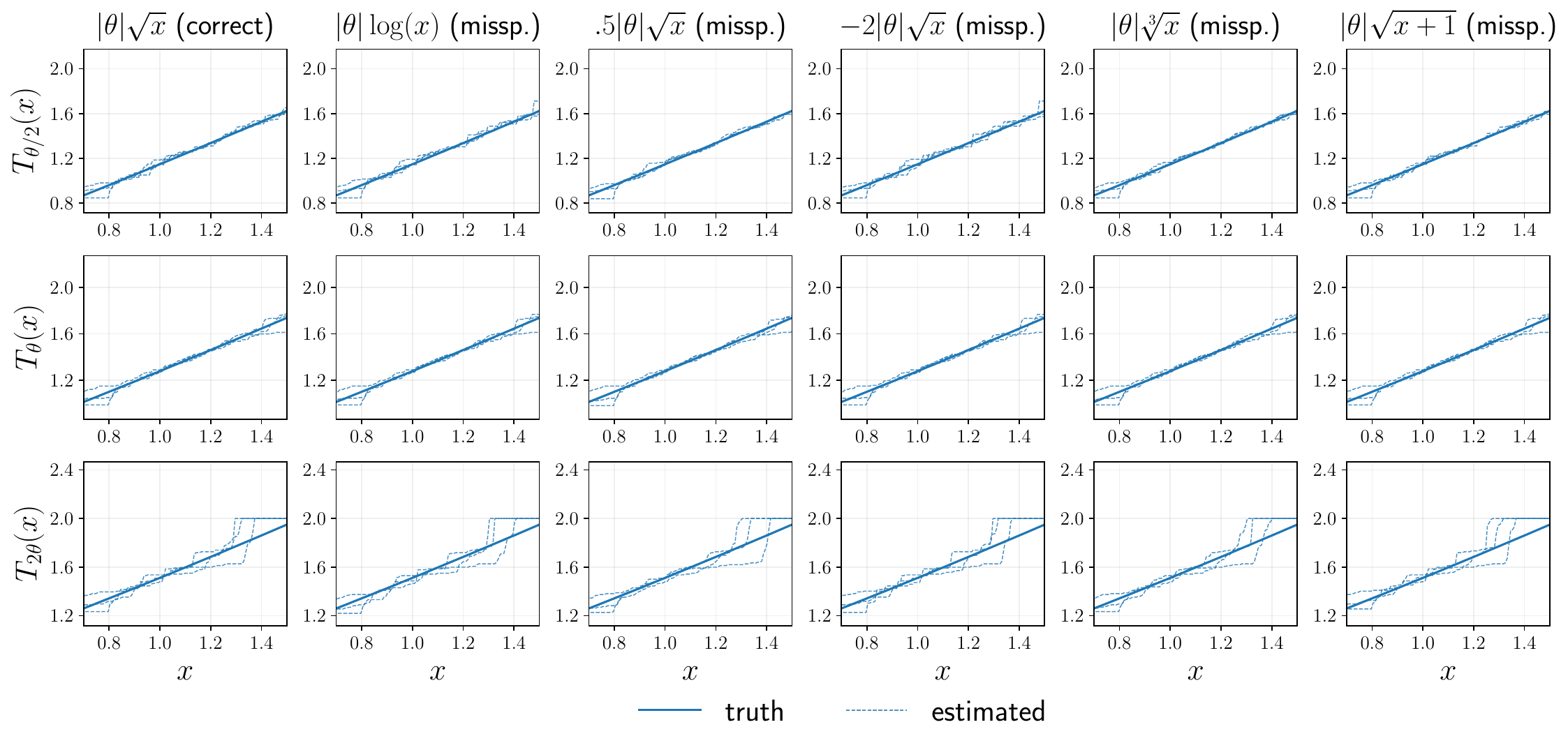}
\vspace{-0.2cm}
\caption{\small The estimation of the transport map $T_{\tilde{\theta}}$ is robust to the misspecification of the benefit function for values of $\tilde{\theta}$ different from those used to induce the ex-post distribution and estimate $\varphi'$, \ie, $\theta$.}
\label{fig:exp1_T_est_append}
\end{figure}

\begin{figure}[H]
\centering
\includegraphics[width=.8\textwidth]{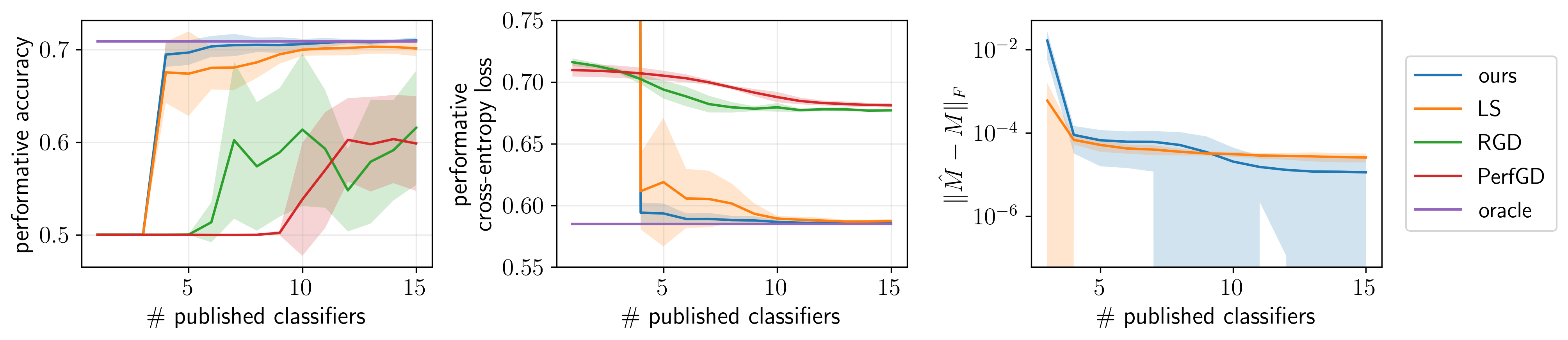}\\
\includegraphics[width=.8\textwidth]{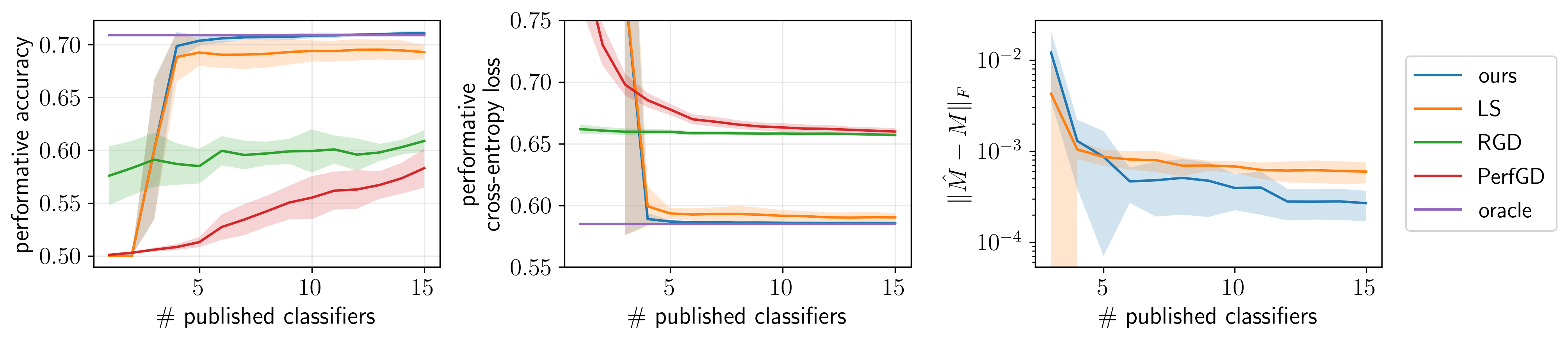}
\caption{Performative performance for different numbers of published classifiers ($M=.01 * I_3$ in the upper plots and $M=.05 * I_3$ in the lower plots).}
\end{figure}



\end{document}